\documentclass[anon,12pt]{conftopaper}

\newif\ifDRAFT
\DRAFTtrue
\ifDRAFT
\newcommand{\marrow}{\marginpar[\hfill$\longrightarrow$]{$\longleftarrow$}}
\newcommand{\niceremark}[3]
   {\textcolor{red}{\textsc{#1 #2:} \marrow\textsf{#3}}}
\newcommand{\ken}[2][says]{\niceremark{Ken}{#1}{#2}}

\newcommand{\michael}[2][says]{\niceremark{Michael}{#1}{#2}}
\newcommand{\michal}[2][says]{\niceremark{Michal}{#1}{#2}}
\newcommand{\feynman}[2][says]{\niceremark{Feynman}{#1}{#2}}
\else
\newcommand{\ken}[1]{}
\newcommand{\michael}[1]{}
\newcommand{\michal}[1]{}
\newcommand{\feynman}[1]{}
\fi
\newcommand{\norm}[1]{{\| #1 \|}}

\def\Nb{{\mathbf{N}}}

\def\S{\mathbf{S}}

\def\val {{\mathrm{val}}}

\def\Deltab{{\mathbf{\Delta}}}

\def\H{\mathbf H}

\def\Q{\mathbf Q}

\ifx\BlackBox\undefined
\newcommand{\BlackBox}{\rule{1.5ex}{1.5ex}}  
\fi
\DeclareMathOperator*{\argmin}{\mathop{\mathrm{argmin}}}

\DeclareMathOperator*{\diag}{\mathop{\mathrm{diag}}}
\def\x{\mathbf x}
\def\y{\mathbf y}

\def\a{\mathbf a}
\def\b{\mathbf b}

\def\e{\mathbf e}
\def\zero{\mathbf 0}

\def\B{\mathbf B}
\def\A{\mathbf A}

\def\U{\mathbf U}

\def\D{\mathbf D}

\def\M{\mathbf M}

\def\I{\mathbf I}

\def\A{\mathbf A}

\def\E{\mathbb E}
\def\R{\mathbb R}
\def\N{\mathbb N}
\def\Pr{\mathrm{Pr}}
\def\tr{\mathrm{tr}}

\let\origtop\top
\renewcommand\top{{\scriptscriptstyle{\origtop}}} 

\definecolor{silver}{cmyk}{0,0,0,0.3}
\definecolor{yellow}{cmyk}{0,0,0.9,0.0}
\definecolor{reddishyellow}{cmyk}{0,0.22,1.0,0.0}
\definecolor{black}{cmyk}{0,0,0.0,1.0}
\definecolor{darkYellow}{cmyk}{0.2,0.4,1.0,0}
\definecolor{orange}{cmyk}{0.0,0.7,0.9,0}
\definecolor{darkSilver}{cmyk}{0,0,0,0.1}
\definecolor{grey}{cmyk}{0,0,0,0.5}
\definecolor{darkgreen}{cmyk}{1,0,1,0} 

\ifx\proof\undefined
\newenvironment{proof}{\par\noindent{\bf Proof\ }}{\hfill\BlackBox\\[2mm]}
\fi

\ifx\theorem\undefined
\newtheorem{theorem}{Theorem}
\fi

\ifx\example\undefined
\newtheorem{example}{Example}
\fi

\ifx\condition\undefined
\newtheorem{condition}{Condition}
\fi
\ifx\property\undefined
\newtheorem{property}{Property}
\fi

\ifx\lemma\undefined
\newtheorem{lemma}{Lemma}
\fi

\ifx\proposition\undefined
\newtheorem{proposition}{Proposition}
\fi

\ifx\remark\undefined
\newtheorem{remark}{Remark}
\fi

\ifx\corollary\undefined
\newtheorem{corollary}{Corollary}
\fi

\ifx\definition\undefined
\newtheorem{definition}{Definition}
\fi

\ifx\conjecture\undefined
\newtheorem{conjecture}{Conjecture}
\fi

\ifx\axiom\undefined
\newtheorem{axiom}{Axiom}
\fi

\ifx\claim\undefined
\newtheorem{claim}{Claim}
\fi

\ifx\assumption\undefined
\newtheorem{assumption}{Assumption}
\fi

\ifx\condition\undefined
\newtheorem{condition}{Condition}
\fi

\newcommand{\aZero}{\alpha_0}
\newcommand{\tZero}{t_0}

\newcommand{\Ls}{L}

\title[Last-Iterate Convergence of Randomized Kaczmarz
  and SGD]{Last-Iterate Convergence of Randomized Kaczmarz
  \\and SGD with Greedy Step Size}
\usepackage{times}
\coltauthor{%
 \Name{Micha{\l} Derezi\'nski} \\\Email{derezin@umich.edu} 
 \addr University of Michigan
\AND
 \Name{Xiaoyu Dong} \\\Email{xdong@nus.edu.sg} 
 \addr National University of Singapore%
}

\begin{document}

\maketitle

\begin{abstract}%
We study last-iterate convergence of SGD with greedy step size over smooth quadratics in the interpolation regime, a setting which captures the classical Randomized Kaczmarz algorithm as well as other popular iterative linear system solvers. For these methods, we show that the $t$-th iterate attains an $O(1/t^{3/4})$ convergence rate, addressing a question posed by Attia, Schliserman, Sherman, and Koren, who gave an $O(1/t^{1/2})$ guarantee for this setting. In the proof, we introduce the family of \emph{stochastic contraction processes}, whose behavior can be described by the evolution of a certain deterministic eigenvalue equation, which we analyze via a careful discrete-to-continuous~reduction.
\end{abstract}


\section{Introduction}
Stochastic Gradient Descent \citep[SGD,][]{robbins1951stochastic} is one of the most extensively studied optimization algorithms. 
This has led to an in-depth understanding of the convergence properties of many variants of SGD across different optimization settings.
It is therefore perhaps surprising that, with such extensive literature, the worst-case convergence remains unresolved for one of the oldest SGD algorithms in one of the most classical settings: the Kaczmarz algorithm for solving consistent systems of linear equations \citep{kaczmarz37}. Given a system of $m$ equations with $n$ unknowns, the Kaczmarz algorithm starts with an initial $n$-dimensional estimate vector, and then iteratively selects one equation to solve, computing the solution that is the closest to its previous estimate. This method has seen renewed interest in the numerical analysis literature since  \cite{strohmer2009randomized} proposed a randomized sampling scheme for selecting the equations. Yet, the worst-case convergence rate (i.e., independent of any condition numbers) is still unknown for this method, regardless of how the equations are sampled, despite recent efforts in this direction \citep{steinerberger2023approximate,evron2025continual,attia2025fast}. 

In the broader context, the above question is directly tied to the study of last-iterate convergence of SGD with fixed step size in the smooth interpolation regime, which has seen significant interest thanks to the effectiveness of this type of algorithms for training highly over-parameterized deep learning models \citep{ma2018power}. Here, the Kaczmarz algorithm can be viewed as an instance of SGD minimizing an average of convex $\beta$-smooth functions, using step size $1/\beta$. This choice of step size is notable, as this is also the canonical choice for full gradient descent (GD) on a $\beta$-smooth function. While SGD is not expected in general to converge under fixed step size, it will do so if all of the averaged functions admit a common minimizer (the \emph{interpolation regime}). In particular, it is known that the average of $t$ such SGD iterates converges at the rate of $O(1/t)$ \citep{bach2013non,zou2021benign}, and after sufficiently shrinking the step size below $1/\beta$, nearly matching rates can be obtained for the last iterate \citep{varre2021last}. Yet, these guarantees do not cover last-iterate SGD with the canonical GD step size $1/\beta$ (called the greedy step size), empirically the most effective choice.

Recent efforts toward the understanding of SGD with greedy step size are motivated not only by the Kaczmarz method (along with its many extensions such as Block Kaczmarz, Coordinate Descent, Sketch-and-Project, etc.), but also by its application to the analysis of catastrophic forgetting in a class of realizable continual learning problems \citep{evron2022catastrophic}. In this context, \cite{evron2025continual} provided an analysis of SGD with greedy step size (including Kaczmarz), showing that its last iterate converges at an $O(1/t^{1/4})$ rate. Later, \cite{attia2025fast} improved this result to $O(1/t^{1/2})$, and asked whether this rate is optimal. (Further related work is provided in Appendix~\ref{a:related-work}.)

\paragraph{Our contributions.} In this work, we provide a new framework for analyzing the last-iterate convergence of SGD algorithms, and we use it to obtain the following main result:
\begin{center}
\emph{The last iterate of SGD over $\beta$-smooth quadratics in the interpolation regime with step size $1/\beta$, including Randomized Kaczmarz \citep{strohmer2009randomized} and Randomized Coordinate Descent \citep{leventhal2010randomized}, attains the $O(1/t^{3/4})$ convergence rate.}   
\end{center}
Curiously, we are able to show that the exponent $3/4$ in the rate is \emph{not} optimal, as our analysis can be pushed further to recover the exponent $3/4+0.001$, however we encounter a fundamental barrier around $3/4+0.003$. Furthermore, our results apply more generally than the canonical methods mentioned above, for example including all linear system solvers based on the so-called Sketch-and-Project framework \citep{gower2015randomized}. In particular, we use our techniques to show that a certain variant of Block Kaczmarz \citep{elfving1980block} attains a stronger worst-case last-iterate convergence guarantee than the classical Kaczmarz method.

\paragraph{Overview of our techniques.} To attain our results, we characterize the convergence of  SGD through what we call a \emph{stochastic contraction process} (Definition \ref{d:process}): a sequence of independent random positive semidefinite (psd) contraction operators applied to a high-dimensional vector. We observe that capturing SGD algorithms with greedy step size involves analyzing such a stochastic process in full generality, without imposing any restrictions (such as upper/lower bounds) on the contraction operators (Theorem~\ref{t:main}). 

We analyze the stochastic contraction process by characterizing it via a deterministic matrix recursion (Lemma \ref{l:mat-rec}). Unfolding this recursion reveals that its spectrum exhibits two regimes: one where the eigenvalues oscillate wildly, and one where they follow a smooth trajectory. We carefully unify these two regimes, and reduce them to a single summation bound (Lemma~\ref{l:main-technical}). Establishing this bound proves remarkably delicate (Section~\ref{s:main-technical}): We achieve this by performing a discrete-to-continuous reduction and analyzing the resulting ordinary differential equation (ODE).

\section{Main Result and Its Implications}

In this section, we present our main result, and then describe its implications for Randomized Kaczmarz and other SGD-type algorithms. To highlight the general nature of the claim, we frame it as a characterization of the behavior of a high-dimensional stochastic process defined by a sequence of independent random psd contraction operators with a common mean ($\preceq$ denotes the Loewner~order). 
\begin{definition}\label{d:process}
    Random sequence $\Deltab_0,\Deltab_1,\Deltab_2,...\in\R^n$ is called a \textbf{stochastic contraction process} with average rate $\bar\M$ if it satisfies $\Deltab_{t+1}=(\I-\M_t)\Deltab_t$ for a sequence of independent random $n\times n$ psd matrices $\M_0,\M_1,\M_2,...$ such that $\zero\preceq\M_t\preceq \I$ and $\E\,\M_t=\bar\M$ for all $t\geq 0$.
\end{definition}
Many stochastic algorithms can be cast as instances of such a process, and many existing convergence arguments can be viewed as analyzing this process under additional restrictions on the contractions (such as bounding them away from zero or from identity). Crucially, our result does not impose any such restrictions. Below and throughout, we use the notation $\|\x\|_{\M} = \sqrt{\x^\top\M\x}$. 
\begin{theorem}\label{t:main}
There are absolute constants $C>0$ and $\theta\geq 0.001$ such that any stochastic contraction process $\{\Deltab_t\}_{t\geq 0}$ with average rate $\bar\M$ satisfies
  \begin{align*}
    \E\,\|\Deltab_t\|_{\bar\M}^2 \leq
    C\cdot\frac{\E\,\|\Deltab_0\|^2}{t^{3/4\,+\,\theta}}.
  \end{align*}
\end{theorem}
\begin{remark}
        The $O(1/t^{3/4\,+\,\theta})$ convergence rate appears to be the best rate attainable via our analysis framework (without introducing further restrictions on $\M_t$) up to $\theta\approx 0.003$ (see Section~\ref{s:optimality}).
\end{remark}

\begin{remark}
    A key feature of Theorem \ref{t:main} is that it allows both $\M_t$ and $\bar\M$ to vary in the full range between zero and identity, which enables it to capture the canonical versions of Randomized Kaczmarz and Randomized Coordinate Descent on worst-case inputs. Restricting $\M_t$ (or its expectation) to a smaller range, e.g., $c_1\I\preceq\M_t\preceq c_2\I$ where either $c_1>0$ or $c_2<1$, leads to simpler results (and potentially faster rates) that are well-known in the literature.
\end{remark}
\subsection{Implications for SGD with Greedy Step Size}
Theorem \ref{t:main} can be interpreted as a convergence guarantee for a stochastic gradient algorithm running on a quadratic function in the interpolation regime. To see this, consider minimizing $f(\x) = \frac1m \sum_{i=1}^m \psi_i(\x)$ over $\x\in\R^n$,
where each $\psi_i$ is a $\beta$-smooth quadratic, i.e., 
$\|\nabla \psi_i(\x) - \nabla\psi_i(\y)\|\leq \beta\|\x-\y\|$ for all $\x,\y\in\R^n$. Moreover, suppose that there exists $\x^*\in\R^d$ that simultaneously minimizes all $\psi_i$ (this is referred to as the \emph{interpolation regime}). 
As a concrete example, we can think of a regression problem defined by $m$ vectors $\a_1,...,\a_m\in\R^n$ and response values $b_1,...,b_m\in\R$, and let $\psi_i(\x) = \frac12(\a_i^\top\x - b_i)^2$ for each $i$. Here, $\beta = \max_i\|\a_i\|^2$, and the interpolation regime occurs when there is a linear model $\a_i\rightarrow \a_i^\top\x^*$ that perfectly fits all of the response values $b_i$, in which case $\psi_i(\x) = \frac12(\a_i^\top(\x-\x^*))^2$ and  $f(\x)=\frac1{2m}\|\x-\x^*\|_{\A^\top\!\A}^2$.

The standard SGD algorithm with fixed step size $\eta$ initialized at $\x_0\in\R^n$ proceeds by randomly sampling one component function at a time and taking a corresponding gradient descent step:
\begin{align*}
    \x_{t+1} = \x_t - \eta \nabla\psi_{i_t}(\x_t),\qquad i_t\sim \{1,...,m\}.
\end{align*}
This can be mapped to a stochastic contraction process via $\Deltab_t = \x_t - \x^*$. Indeed, since $\psi_{i_t}$ is a quadratic minimized by $\x^*$, using the Taylor expansion we have $\nabla\psi_{i_t}(\x_t) = \nabla^2\psi_{i_t}(\x^*) (\x_t-\x^*)$, so by choosing  $\M_t = \eta\nabla^2\psi_{i_t}(\x^*)$ we get $\Deltab_{t+1}=(\I-\M_t)\Deltab_t$. Using $\beta$-smoothness, we have $\M_t\preceq\eta\beta\I$ so the condition $\zero\preceq\M_t\preceq\I$ is satisfied for any $0<\eta\leq 1/\beta$. Since $f(\x_t) - f(\x^*)=\frac1{2\eta}\|\x_t-\x^*\|_{\bar\M}^2$, where $\bar\M = \E\,\M_t = \eta\nabla^2 f(\x^*)$, we obtain the following corollary of Theorem \ref{t:main}.
\begin{corollary}\label{c:sgd}
    Minimizing an average of $\beta$-smooth quadratic functions in the interpolation regime, SGD with step size $1/\beta$ satisfies:
\begin{align*}
    \E\big[f(\x_t) - f(\x^*)\big] 
    = O\bigg(\frac{\beta\|\x_0-\x^*\|^2}{t^{3/4\,+\,\theta}}\bigg).
\end{align*}
\end{corollary}

Here, the fact that our result allows the choice of step size $\eta=1/\beta$ is crucial. For standard GD on a $\beta$-smooth function, $1/\beta$ is the canonical choice of step size \citep{bertsekas2016}. However, when dealing with stochastic gradients, it is common to use either a much smaller fixed step size or a decaying step size schedule in order to compensate for the noise in the convergence analysis \citep{varre2021last,liu2023revisiting}. Yet, in the interpolation regime, the canonical choice of $\eta=1/\beta$ (i.e., the greedy step size) is often empirically the most effective one. Corollary~\ref{c:sgd} continues a recent line of works aiming to close the theory-practice gap in our understanding of SGD with greedy step sizes \citep{evron2025continual,attia2025fast}, improving the rate from $O(1/t^{1/2})$ to $O(1/t^{3/4\,+\,\theta})$.

\subsection{Key Example: Randomized Kaczmarz}

Perhaps the most important application of Theorem \ref{t:main} is in the analysis of randomized iterative methods such as the Kaczmarz algorithm for solving consistent systems of linear equations. Here, we are given an $m\times n$ matrix $\A$ and an $m$-dimensional vector $\b$ such that there exists $\x^*$ satisfying $\A\x^*=\b$. Given an iterate $\x_t$, Kaczmarz chooses one of the $m$ linear equations, $\a_{i_t}^\top\x=b_{i_t}$ (where $\a_i^\top$ denotes the $i$th row of $\A$), and computes $\x_{t+1}$ as the projection of $\x_t$ onto the subspace of the solutions of that equation:
\begin{align*}
    \x_{t+1} = \x_t - \frac{\a_{i_t}^\top\x_t - b_{i_t}}{\|\a_{i_t}\|^2}\a_{i_t}.
\end{align*}
The Kaczmarz algorithm can be viewed as a type of weighted SGD minimizing the $\|\A\x-\b\|^2$ objective \citep{needell2014stochastic}, and thus it analogously maps to Definition \ref{d:process} by setting: 
\begin{align}
    \Deltab_t = \x_t-\x^*\qquad\text{and}\qquad\M_t = \frac{\a_{i_t}\a_{i_t}^\top}{\|\a_{i_t}\|^2},\label{eq:rk-process}
\end{align}
where note that $\M_t$ is simply the rank-1 projection onto the span of $\a_{i_t}$. Naturally, how we select the equation indices $i_t$ has a great impact on the convergence rate of the Kaczmarz algorithm, and \cite{strohmer2009randomized} showed that if we sample $i_t$ with probability proportional to the squared row norm, $\Pr[i_t = i] \propto \|\a_{i}\|^2$, then this Randomized Kaczmarz algorithm will converge to the optimum $\x^*$ at the rate that depends only on the condition number of $\A$, and it requires fewer passes over the matrix than full gradient descent. However, they provide no convergence guarantee that is free of condition number dependence, and despite extensive literature on this subject, the last-iterate convergence rate of Randomized Kaczmarz on worst-case inputs remains unresolved.

Mapping Randomized Kaczmarz to Definition \ref{d:process}, we observe that $\E[\M_t] = \A^\top\A/\|\A\|_F^2$ and $\|\Deltab_t\|_{\bar\M}^2 = \|\A\x_t-\b\|^2/\|\A\|_F^2$, where $\|\A\|_F=\sqrt{\tr(\A^\top\A)}$ denotes the Frobenius norm of $\A$. This yields the following corollary.
\begin{corollary}\label{c:rk}
    For a linear system $\A\x=\b$ with solution $\x^*$, Randomized Kaczmarz satisfies:
    \begin{align*}
        \E\,\|\A\x_t-\b\|^2 = O\bigg(\frac{\|\A\|_F^2\|\x_0-\x^*\|^2}{t^{3/4\,+\,\theta}}\bigg).
    \end{align*}
\end{corollary}
We note that, just like in existing guarantees for weighted SGD \citep{needell2014stochastic}, the use of importance sampling as opposed to uniform sampling allows us to replace the smoothness parameter $\beta=\max_i\|\a_i\|^2$ with an average smoothness $\bar\beta=\frac1m\sum_{i=1}^m\|\a_i\|^2 = \frac1m\|\A\|_F^2\leq \beta$, which is why the above bound has $\|\A\|_F^2$ instead of $m\cdot \max\|\a_i\|^2$ in the numerator. Here, again, our result provides an improvement in the last-iterate convergence rate of Randomized Kaczmarz from the previous $O(1/t^{1/2})$ attained by \cite{attia2025fast} to $O(1/t^{3/4\,+\,\theta})$.

\subsection{Further Implications for Sketch-and-Project Algorithms}

Thanks to its generality, Theorem \ref{t:main} covers a number of other randomized iterative methods for linear systems, including all of those that fall under the framework of Sketch-and-Project, developed by \cite{gower2015randomized}, which in addition to Randomized Kaczmarz also includes Block Kaczmarz \citep{elfving1980block} and Randomized Coordinate Descent \citep{leventhal2010randomized}, among many others. Here, the update is defined via a random $b\times m$ matrix $\S_t$ (the sketching operator) and an $n\times n$ positive definite matrix~$\B$ (which determines the projection norm):
\begin{align}
    \x_{t+1} = \argmin_{\x\in\R^n} \|\x - \x_t\|_\B\quad\text{subject to}\quad \S_t\A\x = \S_t\b.\label{eq:sp}
\end{align}
Under this framework, Randomized Kaczmarz is recovered by letting $\S_t=\e_{i_t}^\top\in\R^{1\times m}$ be a random standard basis row-vector and setting $\B=\I$. Stacking $b>1$ random standard basis row-vectors to produce $\S_t\in\R^{b\times m}$, we recover the Block Kaczmarz method. When $\A$ is a positive definite matrix, then we can consider a different scheme by letting $\B=\A$, which yields the Randomized Coordinate Descent method, proposed by \cite{leventhal2010randomized}. 

To analyze the last-iterate convergence of all of these methods together, we map the general sketch-and-project update \eqref{eq:sp} to a stochastic contraction process as follows:
\begin{align*}
    \Deltab_t = \B^{1/2}(\x_t-\x^*)\qquad\text{and}\qquad \M_t = (\S_t\A\B^{-1/2})^\dagger\S_t\A\B^{-1/2}. 
\end{align*}
Note that here again the matrices $\M_t$ are orthogonal projections, which means that they are between zero and identity, but cannot be bounded away from either, thus requiring the careful convergence analysis framework from Theorem \ref{t:main}. To complete the analysis, observe that 
    $\|\A\x_t- \b\| = \|\Deltab_t\|_{\B^{-1/2}\A^\top\A\B^{-1/2}}$,
so in order to analyze convergence in the residual norm it suffices to bound the matrix $\B^{-1/2}\A^\top\A\B^{-1/2}$ in terms of $\bar\M=\E\,\M_t$ in the Loewner ordering. 

For Randomized Coordinate Descent, this turns out to be straightforward, given the right choice of sampling probabilities. The distribution proposed by \cite{leventhal2010randomized} samples proportionally to the diagonal entries of $\A$, namely $\Pr[\S_t = \e_{i_t}^\top]\propto \A_{i_t,i_t}$. After a simple calculation, this yields $\E\,\M_t = \A/\tr(\A)$. Since $\B^{-1/2}\A^\top\A\B^{-1/2}=\A$, we obtain the following corollary.\footnote{Even though the above discussion assumes that $\B$ is invertible, Corollary \ref{c:rcd} easily extends to any psd linear system.}
\begin{corollary}\label{c:rcd}
    For a psd system $\A\x=\b$ with solution $\x^*$, Randomized Coordinate Descent satisfies:
    \begin{align*}
        \E\,\|\A\x_t-\b\|^2 = O\bigg(\frac{\tr(\A)\|\x_0-\x^*\|_{\A}^2}{t^{3/4\,+\,\theta}}\bigg).
    \end{align*}
\end{corollary}
For Block Kaczmarz, computing the expectation $\E\,\M_t$ explicitly is less straightforward. However, since the projection $\M_t$ only gets larger (in Loewner ordering) when introducing multiple equations, it follows that Block Kaczmarz inherits the guarantee of the corresponding single-row Kaczmarz, such as the one in Corollary~\ref{c:rk} (and same is true for block versions of coordinate descent). Nevertheless, this feels somewhat unsatisfying, since we would hope that the convergence guarantee improves as we increase the block size. In fact, \cite{derezinski2024solving} showed such an improved convergence guarantee when the problem is parameterized by an appropriate notion of condition number and as long as we preprocess the linear system with the Randomized Hadamard Transform \citep[RHT,][]{ailon2009fast}. Here, RHT refers to a random orthogonal transformation\footnote{$\Q=\frac1{\sqrt n}\H\D$, where $\H$ is the $m\times m$ Hadamard matrix and $\D$ is diagonal with i.i.d. Rademacher entries.}
$\Q\in\R^{m\times m}$ which can be applied to $\A$ and $\b$ in $O(mn\log m)$ time and has the property that the transformed system $\Q\A\x=\Q\b$ is equivalent to the original one, but its equations have roughly equal importance. This allows us to use uniform sampling in the Block Kaczmarz algorithm. In the following corollary, we show that when the block size is proportional to the stable rank of $\A$, such RHT-preprocessed Block Kaczmarz satisfies a stronger convergence bound than single-row Kaczmarz, replacing the Frobenius norm with the spectral norm, thus matching full gradient descent up to the convergence exponent.
\begin{corollary}
    Given a system $\A\x=\b$ with solution $\x^*$ and stable rank $r=\|\A\|_F^2/\|\A\|^2$, after preprocessing with the RHT, Block Kaczmarz with block size $b\geq O(r\log r +\log mt)$ satisfies:
    \begin{align*}
        \E\,\|\A\x_t-\b\|^2 = O\bigg(\frac{\|\A\|^2\|\x_0-\x^*\|^2}{t^{3/4\,+\,\theta}}\bigg).
    \end{align*}
\end{corollary}
\begin{proof}
    Lemma 14 of \cite{derezinski2024fine} shows that after RHT preprocessing, $\A\leftarrow \Q\A$, the projection matrix $\M_t=(\S_t\A)^\dagger\S_t\A$ for Block Kaczmarz with $b\geq O(\log (m/\delta))$ satisfies:
    \begin{align*}
        \bar\M = \E\,\M_t \succeq \A^\top(\A\A^\top+\lambda\I)^{-1}\A - \delta\I,
    \end{align*}
    where $\lambda = O\big(\frac{\log b}{b}\|\A\|_F^2\big)$. Choosing $\delta = 1/t$ and $b\geq O(r\log r + \log mt)$ so that $\lambda\leq \|\A\|^2$,
    \begin{align*}
        \E\,\|\A\x_t-\b\|^2 
        &=\E\,\|\x_t-\x^*\|_{\A^\top\A}^2
      \leq\|\A\A^\top+\lambda\I\|\cdot \E\,\|\x_t-\x^*\|_{\bar\M+\delta\I}^2
        \\
        &\leq 2\|\A\|^2\cdot \Big(\E\,\|\x_t-\x^*\|_{\bar\M}^2 + \delta\cdot\E\,\|\x_t-\x^*\|^2\Big)
      = O\bigg(\frac{\|\A\|^2\|\x_0-\x^*\|^2}{t^{3/4+\theta}}\bigg),
    \end{align*}
    where we used that $\E\,\|\x_t-\x^*\|^2\leq \|\x_0-\x^*\|^2$. Noting that RHT preprocessing preserves all of the above norms concludes the proof.
\end{proof}

\subsection{Connections to Averaged Iterate Analysis}

We remark that the main challenge in our analysis of stochastic contraction processes, and their special cases such as SGD with greedy step size and Randomized Kaczmarz, stems from the fact that we seek to bound the last iterate, as opposed to for instance the averaged iterate or a random iterate. To highlight this fact, as an auxiliary result, we provide a simple $O(1/t)$ convergence guarantee for the averaged and random iterates of a stochastic contraction process.
\begin{theorem}\label{t:averaged}
    Given a stochastic contraction process $\{\Deltab_t\}_{t\geq 0}$ with average rate $\bar\M$, let $\tau$ be a random variable uniformly sampled from $\{0,1,...,t\}$, and define $\bar\Deltab_t = \frac1{t+1}\sum_{i=0}^t\Deltab_i$. Then:
    \begin{align*}
        \E\,\|\bar\Deltab_t\|_{\bar\M}^2 \leq \E\,\|\Deltab_\tau\|_{\bar\M}^2 \leq \frac{\E\,\|\Deltab_0\|^2}{t+1}.
    \end{align*}
\end{theorem}
\begin{proof}
    The proof follows a standard argument from averaged iterate analysis of SGD. Observe that:
    \begin{align*}
        \E\,\|\Deltab_{i+1}\|^2 = \E\,\Deltab_i^\top(\I-\M_i)^2\Deltab_i \leq \E\,\|\Deltab_i\|^2 - \E\,\|\Deltab_i\|_{\bar\M}^2.
    \end{align*}
    Summing and canceling out both sides from $0$ to $t$, we get
    \begin{align*}
        \E\,\|\Deltab_{t+1}\|^2 + \sum_{i=0}^t\E\,\|\Deltab_i\|_{\bar\M}^2 \leq \E\,\|\Deltab_0\|^2.
    \end{align*}
    From this, we immediately obtain that
    \begin{align*}
        \E\,\Big\|\frac1{t+1}\sum_{i=0}^t\Deltab_i\Big\|_{\bar\M}^2 \leq 
        \frac1{t+1}\sum_{i=0}^t\E\,\|\Deltab_i\|_{\bar\M}^2 \leq \frac{\E\,\|\Deltab_0\|^2}{t+1},
    \end{align*}
    which concludes the proof.
\end{proof}
Therefore, replacing the last iterate with the averaged/random iterate in each of our examples yields the optimal $O(1/t)$ convergence rate. In particular, using Theorem \ref{t:averaged} to analyze the random iterate $\x_{\tau}$ in Randomized Kaczmarz yields a positive answer to a question posed by \cite{steinerberger2023approximate} about whether there always exists a sequence of $t = O(\|\A\|_F^2/\epsilon^2)$ Kaczmarz updates that attains the guarantee $\|\A\x_t - \b\|\leq \epsilon\|\x_0-\x^*\|$.

\section{Convergence Analysis via Matrix Recursion}
In this section, we characterize the convergence behavior of a stochastic contraction process. We start by upper bounding the expected norm of the random vectors by defining a matrix recursion that captures how the process evolves along the eigendirections of the average rate matrix $\bar\M$.
\begin{lemma}\label{l:mat-rec}
  Given an $n\times n$ psd matrix $\zero\preceq\bar\M\preceq\I$, define the following matrix recursion:
  \begin{align}
    \Nb_0 = \bar\M,\qquad \Nb_{t+1} = \Nb_t(\I-2\bar\M) + \|\Nb_t\|\cdot \bar\M.\label{eq:mat-rec}
  \end{align}
Then, any stochastic contraction process $\{\Deltab_t\}_{t\geq 0}$ with average rate $\bar\M$ satisfies:
\begin{align*}
  \E\,\|\Deltab_t\|_{\bar\M}^2 \leq \E\,\|\Deltab_0\|_{\Nb_t}^2.
\end{align*}
\end{lemma}
\begin{proof}
 Let $\{\M_t\}_{t\geq 0}$ be the sequence of contractions that define $\{\Deltab_t\}_{t\geq 0}$. Recall that $\zero\preceq\M_t\preceq\I$ and $\E\,\M_t=\bar\M$ for
 all $t$. Fix some $t\geq 0$ and define the following sequence for $i=0,...,t$:
  \begin{align*}
    \bar\Nb_{t,0} = \bar\M,\qquad \bar\Nb_{t,i} = \E\big[(\I - \M_{t-i})\bar\Nb_{t,i-1}(\I-\M_{t-i})\big].
  \end{align*}
 Without loss of generality, assume that $\Deltab_0$ is deterministic. Then, $\E\,\|\Deltab_0\|_{\bar\M}^2 = \|\Deltab_0\|_{\bar\Nb_{t,0}}^2$
 and 
 \begin{align*}
   \E\,\|\Deltab_{t}\|_{\bar\M}^2
   &= \E\Big[\Deltab_{t-1}^\top\E\big[(\I-\M_{t-1})\bar\M(\I-\M_{t-1})\big]\Deltab_{t-1}\Big]
   \\
   &=\E\big[\Deltab_{t-1}^\top\bar\Nb_{t,1}\Deltab_{t-1}\big]= ... = \Deltab_0^\top\bar\Nb_{t,t}\Deltab_0.
 \end{align*}
 Next, we show by induction that $\bar\Nb_{t,i}\preceq \Nb_i$ for each
 $i=0,...,t$. Clearly, $\bar\Nb_{t,0}=\Nb_0$, so suppose that
 $\bar\Nb_{t,i-1}\preceq \Nb_{i-1}$ for some $1\leq i\leq t$. Then:
 \begin{align*}
   \bar\Nb_{t,i}
   &\preceq  \E\big[(\I -
   \M_{t-i})\Nb_{i-1}(\I-\M_{t-i})\big]
   \\
   & = \Nb_{i-1} - \E\,\M_{t-i}\Nb_{i-1} - \Nb_{i-1}\E\,\M_{t-i} +
     \E\,\M_{t-i}\Nb_{i-1}\M_{t-i}
   \\
   & = \Nb_{i-1} - \bar\M\Nb_{i-1} - \Nb_{i-1}\bar\M +
     \E\,\M_{t-i}\Nb_{i-1}\M_{t-i}
   \\
   &\preceq \Nb_{i-1}(\I - 2\bar\M) + \|\Nb_{i-1}\|\cdot
     \E\,\M_{t-i}^2    , 
   \\
   &\preceq \Nb_{i-1}(\I - 2\bar\M) + \|\Nb_{i-1}\|\cdot\bar\M,
 \end{align*}
 where we used that by definition $\Nb_{i-1}$
 commutes with $\bar\M$, and then that $\E\,\M_{t-i}^2\preceq
 \E\,\M_{t-i}$.
 Thus, we have shown that $\bar\Nb_{t,i}\preceq\Nb_i$
 for each $i=0,...,t$ and each $t\geq 0$. In particular, this implies
 that $\bar\Nb_{t,t}\preceq\Nb_t$. The claim now follows since
 $\E\,\|\Deltab_{t}\|_{\bar\M}^2=\|\Deltab_0\|_{\bar\Nb_{t,t}}^2\leq \|\Deltab_0\|_{\Nb_t}^2$.
\end{proof}
We next prove Theorem \ref{t:main} by analyzing the evolution of the eigenvalues in recursion \eqref{eq:mat-rec}.
\subsection{Proof of Theorem \ref{t:main}}
Lemma \ref{l:mat-rec} implies that
$\E\,\|\Deltab_{t}\|_{\bar\M}^2\leq \|\Nb_t\|\cdot\E\,\|\Deltab_0\|^2$, so
it now suffices to bound $\|\Nb_t\|$ for matrices $\Nb_t$ defined
by the recursion \eqref{eq:mat-rec}. Since these matrices all commute
with $\bar\M$ (and therefore, with each other), we can rewrite the
recursion purely as a transformation of their eigenvalues in the
common basis. Let $\bar\M = \U\D\U^\top$ be the eigendecomposition of
$\bar\M$, where $\D = \diag(\rho_1,...,\rho_n)$ and $\rho_1\geq
\rho_2\geq ...$ denote the eigenvalues of $\bar\M$. Also, let $\lambda_{k,t}$
be the eigenvalue of $\Nb_t$ associated with the $k$th eigenvector in
$\U$. Then, these eigenvalues are governed by the following set of
recursions:
\begin{align}
  \lambda_{k,t+1} = \lambda_{k,t}\cdot (1 - 2\rho_k) + \rho_k\cdot
  \max_{i}\lambda_{i,t},\qquad k=1,...,n.\label{eq:lambda-rec}
\end{align}
It now suffices to show that $\|\Nb_t\|=\lambda_{\max,t}\leq
\frac{C}{t^{\alpha}}$ for each $t$, where  we define the shorthand $\lambda_{\max,t}
:= \max_k\lambda_{k,t}$, whereas the constant $C\geq 2$
will be chosen later.  We do
this via induction, using the following slightly stronger inductive hypothesis, 
distinguishing between even and odd $t$:
\begin{align*}
    \textbf{Inductive hypothesis:}\qquad \lambda_{\max,t}\leq \begin{cases} \frac{C}{(t+1)^\alpha} &\text{if $t$ is even,}\\
    \frac{C}{(t+2)^\alpha} &\text{if $t$ is odd}.
    \end{cases}
\end{align*}
  
As the base case, consider all $0\leq t\leq C-2$. Then, for all $k$ we have 
\begin{align*}
  \lambda_{k,t}
  &= \lambda_{k,t-1}(1-\rho_k) + (\lambda_{\max,t-1}-\lambda_{k,t-1})\rho_k
\leq \lambda_{\max,t-1}\leq...\leq \lambda_{\max,0}\leq 1\leq \frac{C}{(t+2)^\alpha},
\end{align*}
where we used that $\lambda_{k,t-1}\geq 0$ since $\Nb_{t-1}$ is positive semidefinite.

Next, suppose that the induction hypothesis holds for for all
$0,1,...,t-1$. We will show the hypothesis for $t$. We break the
analysis down
into two subsets of the indices $k$, depending on whether $\rho_k\leq 1/2$ or not, to account for the possibility of changing signs in the first term of \eqref{eq:mat-rec}.

\paragraph{Case 1: $\rho_k\leq 1/2$.} We start by
focusing only on $k$ such that $\rho_k\leq 1/2$. Here, we can treat
the even and odd $t$ cases together by showing that $\lambda_{k,t}\leq
\frac{C}{(t+2)^\alpha}$. We start by expanding the recursion:
\begin{align*}
  \lambda_{k,t}
  &= \lambda_{k,t-1}(1-2\rho_k) +
                  \rho_k\lambda_{\max,t-1}
\\
&= \rho_k(1-2\rho_k)^t + \rho_k\sum_{i=1}^t(1-2\rho_k)^{t-i}\lambda_{\max,i-1}
  \\
  &\leq \rho_k(1-2\rho_k)^t + C\rho_k\sum_{i=1}^t\frac{(1-2\rho_k)^{t-i}}{i^\alpha}.
\end{align*}
In the last inequality we use the assumption that $\rho_k\leq1/2$ to ensure that
$1-2\rho_k\geq 0$ and all of the terms in the expression are
non-negative. It now suffices to show that the above summation 
formula is bounded by $\frac{C}{(t+2)^\alpha}$, which is obtained
in the following key technical lemma (see Section~\ref{s:main-technical}).
\begin{lemma}\label{l:main-technical}
  For all $K,t\geq 1000$, $\rho\in(0,1/2]$, and
  $\alpha = 3/4+0.001$, we have
  \begin{align*}
    \rho(1-2\rho)^t +
    K\rho\sum_{i=1}^t\frac{(1-2\rho)^{t-i}}{i^\alpha} \leq \frac{K}{(t+2)^\alpha}.
  \end{align*}
\end{lemma}
Despite being entirely elementary, Lemma \ref{l:main-technical} turns out to require a quite involved and technical analysis. In Section \ref{s:main-technical}, we give a proof sketch of a slightly weaker claim with $\alpha=3/4$ (which is still highly technical, but much more readable), and the complete argument is given in the~Appendix.
\paragraph{Case 2: $\rho_k> 1/2$.} Next, we consider only $k$ such
that $\rho_k> 1/2$. In this case, $1-2\rho_k$ is negative, which leads the
recursion to oscillate up and down between even and odd $t$'s (see Figure \ref{f:mat-rec}). We thus
consider these two sub-cases.

\paragraph{Case 2a: Even $t$.} Note that from the inductive hypothesis we
have $\lambda_{\max,t-1}\leq \frac{C}{(t+1)^\alpha}$, since $t-1$ is
odd. Furthermore, as shown before, we have $\lambda_{k,t}\leq \lambda_{\max,t-1}$, so the claim follows.

\paragraph{Case 2b: Odd $t$.} We again expand the recursion, but in a
slightly different way:
\begin{align*}
  \lambda_{k,t}
  &= \lambda_{k,t-1}(1-2\rho_k) + \rho_k\lambda_{\max,t-1}
  \\
  &=\big(\lambda_{k,t-2}(1-2\rho_k) + \rho_k\lambda_{\max,t-2}\big)(1-2\rho_k) +
    \rho_k\lambda_{\max,t-1}
  \\
  &=\lambda_{k,t-2}(1-2\rho_k)^2 + 2\rho_k(1-\rho_k)\lambda_{\max,t-1} +
    \rho_k(2\rho_k-1)(\lambda_{\max,t-1} - \lambda_{\max,t-2})
  \\
  &\leq \lambda_{k,t-2}(1-2\rho_k)^2 + 2\rho_k(1-\rho_k)\lambda_{\max,t-2}, 
\end{align*}
where in the last inequality we used again that $\lambda_{\max,t-1}\leq
\lambda_{\max,t-2}$. In order to recover an analog of Case 1, we substitute $\beta_k = 2\rho_k(1-\rho_k)\in[0,1/2]$, observing that $(1-2\rho_k)^2 = 1-2\beta_k$, and then we re-index the recursion to go over only odd indices $t$. Defining $\gamma_{k,i} = \lambda_{k,2i+1}$ and $\gamma_{\max,i} = \lambda_{\max,2i+1}$ and letting $t=2s+1$, 
we obtain:
\begin{align*}
  \lambda_{k,t} = \gamma_{k,s}
  &\leq \gamma_{k,s-1}(1-2\beta_k) +
  \beta_k\gamma_{\max,s-1}
\leq \gamma_{k,0}(1-2\beta_k)^{s} + \beta_k\sum_{i=1}^s(1-2\beta_k)^{s-i}\gamma_{\max,i-1}.
\end{align*}
Note that we can bound $\gamma_{k,0}=\lambda_{k,1}$ as follows:
\begin{align*}
  \lambda_{k,1}
  &= \lambda_{k,0}(1-2\rho_k) + \rho_k\lambda_{\max,0}
  \leq \rho_k(1-2\rho_k) + \rho_k = 2\rho_k(1-\rho_k)=\beta_k.
\end{align*}
Using the inductive hypothesis, $\gamma_{\max,i-1}=\lambda_{\max,2i-1}\leq \frac{C}{(2i)^\alpha}$, so we
conclude:
\begin{align*}
  \lambda_{k,t}
  &\leq \beta_k(1-2\beta_k)^s + (C/2^\alpha)\beta_k\sum_{i=1}^s\frac{(1-2\beta_k)^{s-i}}{i^\alpha}\leq \frac{C/2^\alpha}{(s+2)^\alpha}\leq \frac{C}{(t+2)^\alpha},
\end{align*}
where in the second-to-last step we again used Lemma
\ref{l:main-technical}, choosing $C = 2^\alpha \cdot1000$ and $\rho=\beta_k$. Thus, we have shown the inductive hypothesis,
which concludes the proof of Theorem \ref{t:main}.

\subsection{(Near-)Optimality of the Recursion Analysis}
\label{s:optimality}

Given the peculiarity of the exponent $3/4+0.001$ in the convergence rate shown above, it is natural to ask what is the minimax optimal rate achieved by the matrix recursion defined in Lemma~\ref{l:mat-rec}.  To examine this question, we first give a nearly matching lower bound construction which shows that the exponent cannot be improved beyond $3/4+0.003$.
\begin{theorem}\label{t:lower}
    For any $T\geq 1$, there is a psd matrix $\zero\preceq \bar\M\preceq \I$ such that the sequence
    \begin{align*}
    \Nb_0 = \bar\M,\qquad \Nb_{t+1} = \Nb_t(\I-2\bar\M) + \|\Nb_t\|\cdot \bar\M
    \end{align*}
    satisfies $\|\Nb_t\| \geq c/t^{3/4\,+\,0.003}$ for every $t=1,2,...,T$, where $c>0$ is an absolute constant.
\end{theorem}
\paragraph{Proof sketch.} The key idea is to first show that the inequality in the summation bound (Lemma~\ref{l:main-technical}) does not hold with $\alpha = 3/4+0.003$. This follows essentially by showing the opposite inequality for some sufficiently small $\rho>0$. We do this using a discrete-to-continuous reduction that, while different from the argument described in Section \ref{s:main-technical}, has a similar flavor. Then, we  extend this lower bound to the deterministic recursion by constructing a sufficiently large and ill-conditioned matrix $\bar\M$ so that it has eigenvalues that densely cover the interval $[\rho,1]$, forcing the recursion to exhibit the worst-case behavior represented by the summation in Lemma \ref{l:main-technical}. See Appendix \ref{app:lower-bound} for~details.

\begin{figure}
\hspace{-9mm}\includegraphics[width=0.575\textwidth]{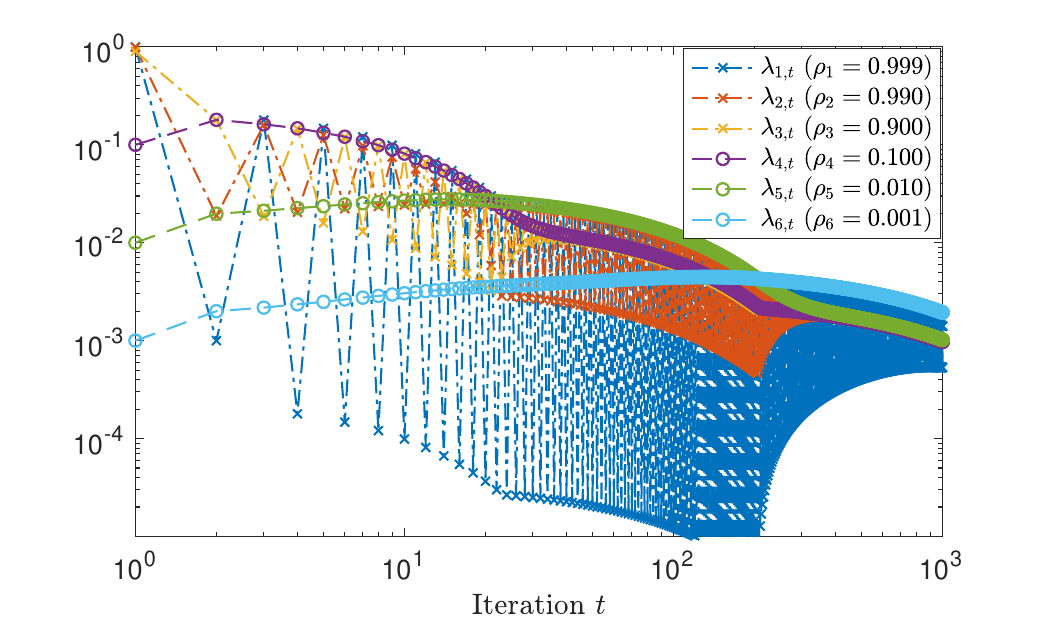}%
  \hspace{-6mm}\includegraphics[width=0.575\textwidth]{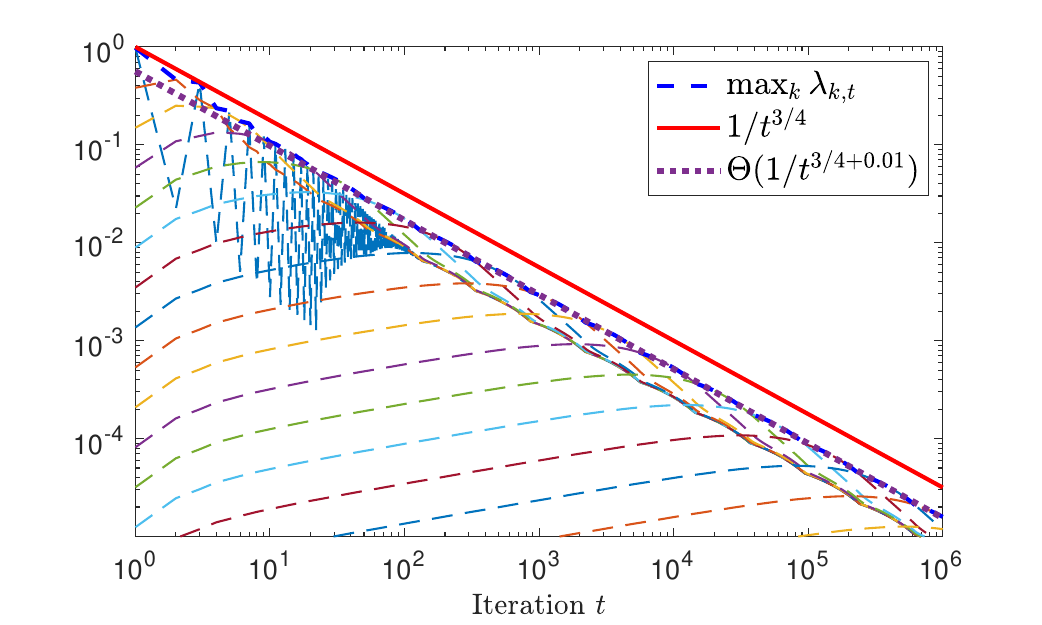}
  \caption{Two simulations of the eigenvalue recursion \eqref{eq:lambda-rec}. On the left is a small example ($n=6$) illustrating the two regimes that distinguish the behavior of eigenvalues below and above the 1/2 threshold. On the right is a larger example consisting of $n=50$ eigenvalues spread evenly in the log-scale between 1 and $10^{-20}$. Here, the dotted line is a best linear fit on the log-log plot for the evolution of $\max_k \lambda_{k,t}$ (the exponent $3/4+0.01$ was rounded). To best accommodate the log-log plot we start indexing the recursions at $t=1$, not at 0.}
  \label{f:mat-rec}
\end{figure}

\paragraph{Numerical simulations.}
To further illustrate this phenomenon, we performed numerical simulations of the recursion using a number of initial conditions (defined by the eigenvalues $\rho_k$ of a hypothetical matrix $\bar\M$). The results of these simulations exhibit a convergence behavior that aligns closely with our theory.

In Figure \ref{f:mat-rec}, we show the results of two simulations by plotting the evolution of all of the eigenvalues of the matrices $\Nb_t$. The left plot is a smaller example that is designed to highlight the two distinct convergence regimes that arise in our analysis, depending on whether the initial eigenvalue is below or above $1/2$. Here, we consider matrices of dimension $n=6$, with the first three eigenvalues close to 1, and the latter three close to 0. As we can see, the initially large eigenvalues exhibit an oscillating behavior as the convergence progresses, whereas the initially small eigenvalues exhibit a more smooth trajectory. This is because the eigenvalue recursion \eqref{eq:lambda-rec} encounters a sign-change in the term $1-2\rho_k$ if and only if the initial eigenvalue $\rho_k$ is larger than $1/2$. This is handled in the proof by separating out these two cases. Notably, in the end both cases rely on the same summation bound, but it is applied in a different way each time.

In order to demonstrate a convergence rate close to the minimax rate, we construct a larger example with $n=50$ eigenvalues that exhibit fast exponential decay (Figure \ref{f:mat-rec}, right). As we can see, the maximum eigenvalue decays at a rate that nearly matches $O(1/t^{3/4})$, but a closer examination shows that the actual rate is slightly better, at around $\alpha \approx 3/4+0.01$. We note that the gap between this rate and the minimax lower bound of $\alpha \approx 3/4+0.003$ from Theorem \ref{t:lower} is likely caused by the small size of the matrices we use in our simulation (due to computational constraints) compared to the matrices we use in the lower bound construction, which are much~larger.

\section{Proof Sketch of the Main Summation Bound
}
\label{s:main-technical}

In this section, we give a sketch of the proof of the main technical lemma (Lemma \ref{l:main-technical}) used to establish Theorem \ref{t:main}. For the sake of clarity, we overview the argument needed to obtain a slightly modified claim, with $\alpha=3/4$ instead of $\alpha=3/4+0.001$, but with a better constant, $300$ instead of $1000$. The remaining details for this simplified claim are in Appendices \ref{a:bounds-for-a}-\ref{app:cert}, whereas the remaining details of the full proof of Lemma \ref{l:main-technical} with $\alpha=3/4+0.001$ is in Appendix \ref{app:main-tech-full}.

\paragraph{Declaration of LLM assistance.} An LLM-based assistant was used during the development of the proof of Lemma \ref{l:main-technical}. The authors verified all steps and take responsibility for the result.

\paragraph{Modified Lemma \ref{l:main-technical}.} We will show that for all $K,t\geq 300$, $\rho\in(0,1/2]$, and
  $\alpha = 3/4$, we have
  \begin{align*}
    \rho(1-2\rho)^t +
    K\rho\sum_{i=1}^t\frac{(1-2\rho)^{t-i}}{i^\alpha} \leq \frac{K}{(t+2)^\alpha}.
  \end{align*}
In our argument, we will treat the two terms on the left hand side separately. To that end, define
\begin{align*}
A_t(\rho):=\rho(1-2\rho)^t \quad\text{and}\quad
B_t(\rho,\alpha):=\rho\sum_{i=1}^t \frac{(1-2\rho)^{t-i}}{i^\alpha}.
\end{align*}
Multiplying both terms by $(t+2)^\alpha$, the inequality is now equivalent to
\begin{align*}
    (t+2)^\alpha A_t(\rho) +
    K(t+2)^\alpha B_t(\rho,\alpha) \leq K.
  \end{align*}
We will next bound the suprema of both terms over all admissible values of $t$ and $\rho$: 
\begin{equation}\label{eq:env}
A_{\sup}(\alpha):=\sup_{t\ge 300}\ \sup_{\rho\in(0,1/2]}\ (t+2)^\alpha A_t(\rho),
\qquad
B_{\sup}(\alpha):=\sup_{t\ge 300}\ \sup_{\rho\in(0,1/2]}\ (t+2)^\alpha B_t(\rho,\alpha).
\end{equation}
The following lemma provides a simple characterization of the constant $K$ that can satisfy our inequality.
\begin{lemma}\label{prop:struct}
Fix $\alpha>0$. Assume $A_{\sup}(\alpha)<\infty$ and $B_{\sup}(\alpha)<1$.
Then for any $K>0$ satisfying
\begin{equation}\label{eq:Cchoice}
K\ge \frac{A_{\sup}(\alpha)}{1-B_{\sup}(\alpha)}
\end{equation}
we have
\begin{align*}
    \rho(1-2\rho)^t +
    K\rho\sum_{i=1}^t\frac{(1-2\rho)^{t-i}}{i^\alpha} \leq \frac{K}{(t+2)^\alpha}
  \end{align*}
for all $t\geq 300$, $\rho\in(0,1/2]$.
\end{lemma}

\begin{proof}
If $K$ satisfies \eqref{eq:Cchoice}, then $A_{\sup}(\alpha)\le K(1-B_{\sup}(\alpha))$, so
\[
A_{\sup}(\alpha)+K\,B_{\sup}(\alpha)\le K(1-B_{\sup}(\alpha))+K\,B_{\sup}(\alpha)=K.
\]
Therefore, we have \begin{align*}
    \sup_{t\ge 300}\sup_{\rho\in(0,\tfrac12]}((t+2)^\alpha A_t(\rho)+K (t+2)^\alpha B_t(\rho))\le K
\end{align*}
which implies the desired result.
\end{proof}

The first term $A_{\sup}(\alpha)$ is easier to bound because we can find the exact maximizers $\rho$ and $t$ that attain the value $A_{\sup}(\alpha):=\sup_{t\ge 300}\ \sup_{\rho\in(0,1/2]}\ (t+2)^\alpha A_t(\rho)$ (see Lemma~\ref{lem:A-max} and Lemma~\ref{lem:A-bound} for details), and eventually we obtain
\[
A_{\sup}(3/4)\le \frac{363307}{7228832}<0.051.
\]

To bound the second term, which is a discrete sum, we perform a discrete-to-continuous reduction by comparing it with an integral. For simplicity, we will replace the $(t+2)^\alpha$ factor with $t^\alpha$, which can be easily addressed later for sufficiently large $t$. Then, letting  $\delta:=-\log(1-2\rho) \geq 0$, we have:
\begin{align*}
    (1-2\rho)^{t-i} = e^{-\delta(t-i)}.
\end{align*}
Therefore, the quantity $t^\alpha B_t(\rho,\alpha)$ corresponds to the integral $\rho\,t\int_0^1 e^{-\delta t(1-u)}u^{-\alpha}\,du$, which after a simple change of variable, $\theta := \delta t/2$ (see Remark~\ref{rem:motL}), results in the following function for $\theta>0$:
\begin{equation}\label{eq:L}
L_\alpha(\theta):=\theta\int_0^1 e^{-2\theta u}(1-u)^{-\alpha}\,du.
\end{equation}
This way, the integral comparison reduces the two-parameter family $(t,\rho)$ to the one-parameter family $\theta$.
Formally, by Lemma~\ref{lem:Bt-dom} and Lemma~\ref{lem:Bt-endpoint}, for any $0<\alpha<1$, $t\ge1$, and $\rho\in(0,1/2]$, we have the following comparison:
\begin{equation}\label{eq:Bt-dom-1}
t^\alpha B_t(\rho,\alpha)\le \max\bigg\{L_\alpha\Bigl(\frac{t}{2}\bigl(-\log(1-2\rho)\bigr)\Bigr),1/2\bigg\}.
\end{equation}

Therefore, it suffices to bound $\sup_{\theta>0}L_\alpha(\theta)$. To this end, we first observe the following ordinary differential equation (ODE) property of $L_{\alpha}(\theta)$,
\begin{equation}\label{eq:ODE}
L_\alpha'(\theta)=1-\Bigl(2-\frac{\alpha}{\theta}\Bigr)L_\alpha(\theta),
\qquad \theta>0.
\end{equation}
(see Lemma~\ref{prop:ODE} for details). This ODE property implies that we can bound $\sup_{\theta>0}L_\alpha(\theta)$ by checking a single point. More precisely, we have the following result.
\begin{lemma}[One-point criterion]\label{lem:onepoint-1}
For any $0<\alpha<1$ and $\ell\in(1/2,1)$, define
\begin{equation}\label{eq:theta-ell-1}
\theta_\ell:=\frac{\alpha}{\,2-\frac1\ell\,}.
\end{equation}
If $L_\alpha(\theta_\ell)<\ell$, then $L_\alpha(\theta)<\ell$ for all $\theta>0$, and therefore $\sup_{\theta>0}L_\alpha(\theta)\le \ell$.
\end{lemma}

\noindent \textbf{Proof sketch} \ Here we explain the key intuition and ideas of the proof, while the full proof is in Appendix~\ref{appen: propL}. Assume for contradiction that there exists $\theta_0>0$ with $L_\alpha(\theta_0)\ge \ell$.
By standard limit calculation, we have
$
0\le L_\alpha(\theta)
\to 0, \text{ as } \theta \to 0
$.
Therefore, $L_\alpha(\theta)<\ell$ for all sufficiently small $\theta>0$.
By continuity of $L_\alpha$ , the set
$\{\theta>0:\ L_\alpha(\theta)=\ell\}$ is nonempty, and the first hitting time
\[
\tau:=\inf\{\theta>0:\ L_\alpha(\theta)=\ell\}
\]
is therefore well defined.
Then $\tau>0$, $L_\alpha(\theta)<\ell$ for all $\theta\in(0,\tau)$, and $L_\alpha(\tau)=\ell$.
Using a standard argument from calculus (Lemma~\ref{lem:hitting-below}), this implies that $L_\alpha'(\tau)\ge 0$.

Next, using the ODE property, we make the following key observation about the structure  of $L_\alpha'(\theta)$ when $L_\alpha$ passes the $y=\ell$ threshold. Define for $\theta>0$ the scalar function
\[
\Psi(\theta):=1-\Bigl(2-\frac{\alpha}{\theta}\Bigr)\ell.
\]
Since $L_\alpha$ satisfies the ODE \eqref{eq:ODE}, at any $\theta>0$ with $L_\alpha(\theta)=\ell$ we have
\begin{equation*} 
L_\alpha'(\theta)=\Psi(\theta).
\end{equation*}
Also, a simple calculation shows that
\begin{equation*} 
\Psi(\theta)>0 \ \text{ for }\ \theta<\theta_\ell,
\qquad
\Psi(\theta)=0 \ \text{ for }\ \theta=\theta_\ell,
\qquad
\Psi(\theta)<0 \ \text{ for }\ \theta>\theta_\ell.
\end{equation*}

Note that we cannot conclude that $\theta_\ell$ is global maximizer because $L_\alpha'(\theta)=\Psi(\theta)$ only holds for $L_\alpha(\theta)=\ell$. However, the above observation shows that every time when $L_\alpha$ passes the line $y=\ell$ before $\theta_\ell$, the function must be increasing, so it will not be able to go down and satisfy our assumption $L_\alpha(\theta_\ell)<\ell$. This is sufficient to attain contradiction. \hfill\BlackBox
\bigskip

Having shown Lemma~\ref{lem:onepoint-1}, we just need to check
\[
L_{3/4}\Bigl(\theta_{\ell_0}\Bigr)<\ell_0,
\]
where $\ell_0=\frac{497}{500}$ (see Lemma~\ref{lem:L-onepoint} for details) and then conclude that $B_{\sup}(3/4)\le \frac{201}{200}\cdot\frac{497}{500}=\frac{99897}{100000}<1$ (see Lemma~\ref{lem:large-t} for details).

Finally, combining 
$A_{\sup}(3/4)\le 0.051$ and 
$
B_{\sup}(3/4)\le \frac{99897}{100000}
$, we have
\[
\frac{A_{\sup}(3/4)}{1-B_{\sup}(3/4)}
\le \frac{\frac{51}{1000}}{\frac{103}{100000}}
= \frac{51}{1000}\cdot\frac{100000}{103}
= \frac{5100}{103}
< 50 < 300.
\]
Therefore, using
Lemma~\ref{prop:struct}, we obtain the desired claim.

\section{Conclusions}
We gave a new last-iterate convergence analysis for a class of SGD algorithms with greedy step size, which includes Randomized Kaczmarz and Randomized Coordinate Descent, among many others, showing that they attain an $O(1/t^{3/4})$ expected convergence rate. Our results provide a direct improvement over the previously known $O(1/t^{1/2})$ guarantee for this setting and have implications for the analysis of catastrophic forgetting in realizable continual learning problems.

\acks{MD was supported in part by NSF CAREER Grant CCF-233865 and a Google ML and Systems Junior Faculty Award. This work was done in part while MD was visiting the Simons Institute for the Theory of Computing. The authors thank Shabarish Chenakkod, Yuji Nakatsukasa, Elizaveta Rebrova, and Mark Rudelson for helpful conversations.}

\bibliography{pap}

\appendix

\section{Additional Related Work}
\label{a:related-work}

In this section, we overview the most relevant prior work related to the convergence analysis of SGD and Randomized Kaczmarz.

\paragraph{Last-iterate convergence of SGD.} Extensive work has been dedicated to last-iterate analysis of SGD, particularly in the general convex and Lipschitz setting \citep{shamir2013stochastic,jain2019making,harvey2019tight,liu2023revisiting,zamani2025exact}. These results essentially match the corresponding guarantees attained by the average iterate in the convex Lipschitz setting. The last-iterate convergence of SGD has also been extensively studied in the interpolation regime \citep{ge2019step,vaswani2019fast,berthier2020tight,varre2021last,wu2022last,liu2023aiming}, which is motivated by modern over-parameterized machine learning models \citep{ma2018power}. However, until the work of \cite{evron2025continual,attia2025fast}, none of these prior results captured last-iterate SGD with the greedy step size, or the Randomized Kaczmarz method.

\paragraph{Connections to continual learning models.} SGD with greedy step size is closely connected with the analysis of continual learning models, and particularly the phenomenon of catastrophic forgetting \citep{mermillod2013stability}. A number of theoretical works have shown that catastrophic forgetting can be mitigated by using randomization \citep{evron2022catastrophic,swartworth2023nearly,cai2024last}. Most notably, \cite{evron2022catastrophic,evron2025continual} showed that convergence bounds for SGD with greedy step size in the smooth interpolation regime can be used to provide bounds for continual linear regression in the realizable setting. Our results, such as Corollary \ref{c:sgd}, can be directly applied to their framework, improving on those guarantees. We note that an alternative approach has been developed by \cite{levinstein2025optimal} that attains optimal bounds for continual linear regression without relying on SGD with greedy step size.

\paragraph{Randomized linear system solvers.} The last-iterate convergence analysis of randomized iterative methods for solving linear systems, such as Randomized Kaczmarz \citep{strohmer2009randomized}, Randomized Coordinate Descent \citep{leventhal2010randomized}, or Sketch-and-Project \citep{gower2015randomized}, has focused primarily on guarantees that depend on some form of a condition number of the problem \citep{derezinski2024recent}, which effectively corresponds to the strongly convex setting in SGD analysis. For example, in the case of Randomized Kaczmarz, \cite{strohmer2009randomized} established convergence of the form $\E\,\|\x_t-\x^*\|^2\leq (1-\rho)^t\|\x_0-\x^*\|^2$ where $\rho = \|\A\|_F^2/\sigma_{\min}^2(\A)$. Faster rates are known for Block Kaczmarz \citep{derezinski2024solving,derezinski2025randomized} and other instances of sketch-and-project \citep{derezinski2022sharp,askotch}, but all of them become vacuous as $\sigma_{\min}(\A)$ tends to zero. \cite{steinerberger2023approximate} studies conditioning-free convergence of a non-constructive Kaczmarz procedure that attains an $O(\log(t)/t)$ rate but only for the optimal sequence of equation selections. As an auxiliary result, we show that a constructive Kaczmarz procedure attains an $O(1/t)$ rate, by choosing a random iterate (instead of the last iterate) in Randomized Kaczmarz.

\section{Calculus Preliminaries}
In this section, we collect auxiliary lemmas that describe standard results from calculus which are useful in our proof of Lemma \ref{l:main-technical}.

\begin{lemma}[Bernoulli's inequality for $0<\alpha\le 1$]\label{lem:bernoulli}
If $0<\alpha\le 1$, then for all $u\ge0$,
\[
(1+u)^\alpha\le 1+\alpha u.
\]
\end{lemma}

\begin{lemma}[Tangent and quadratic lower bounds for $e^x$]\label{lem:exp-bounds}

(i) For all $x\in\R$, $e^x\ge 1+x$.

(ii) For all $x\ge 0$, $e^x\ge 1+x+\frac{x^2}{2}$.
\end{lemma}

\begin{lemma}\label{lem:logineq}
    For every real number $x$ with $x>0$ or $x<-1$,
	\[
	\log\!\Bigl(1+\frac{1}{x}\Bigr)\ge \frac{1}{x+1}.
	\]
\end{lemma}

\begin{lemma}[Derivative at a first hitting time from below]\label{lem:hitting-below}
Let $f:(0,\infty)\to\R$ be continuous and differentiable at $\tau>0$. Fix $\ell\in\R$ and assume
\[
\tau=\inf\{\theta>0:\ f(\theta)=\ell\},
\qquad\text{and}\qquad
f(\theta)<\ell\ \text{ for all }\theta\in(0,\tau).
\]
Then $f'(\tau)\ge0$.
\end{lemma}

\begin{proof}
For $h>0$ small, $f(\tau-h)<\ell=f(\tau)$, so
\[
\frac{f(\tau)-f(\tau-h)}{h}\ge0.
\]
Since $f$ is differentiable at $\tau$, the desired result holds by definition of derivative.
\end{proof}

\begin{lemma}[Derivative at a first return time from above]\label{lem:hitting-above}
Let $f:(0,\infty)\to\R$ be continuous and differentiable at $\sigma>0$. Fix $\ell\in\R$ and assume
\[
\sigma=\inf\{\theta>\tau:\ f(\theta)=\ell\}
\]
for some $\tau\ge0$, and $f(\theta)>\ell$ for all $\theta\in(\tau,\sigma)$.
Then $f'(\sigma)\le0$.
\end{lemma}

\begin{proof}
For $h>0$ small, $f(\sigma-h)>\ell=f(\sigma)$, so
\[
\frac{f(\sigma)-f(\sigma-h)}{h}\le0.
\]
Since $f$ is differentiable at $\sigma$, the desired result holds by definition of derivative.
\end{proof}

\section{Bounds for $A_{\sup}(\alpha)$}
\label{a:bounds-for-a}
In this section, we provide the intermediate results for bounding the $A_t(\rho)$, which is the easier of the two terms analyzed in the proof of Lemma~\ref{l:main-technical}.

\begin{lemma}[Exact maximizer of $A_t$]\label{lem:A-max}
Fix $t\ge1$. The function $\rho\mapsto A_t(\rho)=\rho(1-2\rho)^t$ on $(0,\tfrac12]$
attains its maximum at $\rho^\ast=\frac{1}{2(t+1)}$, with
\[
\sup_{\rho\in(0,1/2]}A_t(\rho)
=\frac{1}{2(t+1)}\Bigl(\frac{t}{t+1}\Bigr)^t.
\]
\end{lemma}

\begin{proof}
For $\rho\in(0,\tfrac12)$, direct differentiation shows that
\[
A_t'(\rho)=(1-2\rho)^t+\rho\cdot t(1-2\rho)^{t-1}(-2)
=(1-2\rho)^{t-1}\bigl(1-2(t+1)\rho\bigr).
\]
Since $(1-2\rho)^{t-1}>0$ on $(0,\tfrac12)$, the unique critical point is $\rho^\ast=\frac{1}{2(t+1)}$.
In addition, $A_t(\rho)\to0$ as $\rho\to 0$ and $A_t(\tfrac12)=0$, and therefore the unique critical point is the global maximizer.
Evaluating $A_t$ at $\rho^\ast$ gives the stated formula.
\end{proof}
In the following lemma, we provide a bound for $A_{\sup}(\alpha)$ that is used in the analysis of the $\alpha=3/4$ version of Lemma \ref{l:main-technical} that is described in Section \ref{s:main-technical}.
\begin{lemma}[Bound for $A_{\sup}(\alpha)$ for $t\ge 300$]\label{lem:A-bound}
Let $\alpha_0=\tfrac34$. Then
\[
A_{\sup}(\alpha_0)\le \frac{363307}{7228832}<0.051.
\]
\end{lemma}

\begin{proof}
Fix an integer $t\ge 300$.
By Lemma~\ref{lem:A-max}, we have
\begin{equation}\label{eq:At-sup-rho}
\sup_{\rho\in(0,\tfrac12]}(t+2)^{\alpha_0}A_t(\rho)
=\frac{(t+2)^{3/4}}{2(t+1)}\Bigl(\frac{t}{t+1}\Bigr)^t.
\end{equation}
Define, for real $x>0$,
\begin{equation}\label{eq:def-fx}
f(x):=\frac{(x+2)^{3/4}}{2(x+1)}\Bigl(\frac{x}{x+1}\Bigr)^x.
\end{equation}
Then \eqref{eq:At-sup-rho} reads $\sup_{\rho}(t+2)^{\alpha_0}A_t(\rho)=f(t)$.

Since $f(x)>0$ for $x>0$, we define
\[
g(x):=\log f(x)\qquad (x>0).
\]
We will compute $g'(x)$ and show $g'(x)<0$ for all $x>0$.
Direct differentiation show that  for every $x>0$,
\begin{align}
g'(x)
&=\frac{3}{4(x+2)}-\frac{1}{x+1}
+\log\!\Bigl(\frac{x}{x+1}\Bigr)+\frac{1}{x+1}\notag\\
&=\frac{3}{4(x+2)}+\log\!\Bigl(\frac{x}{x+1}\Bigr)
=\frac{3}{4(x+2)}-\log\!\Bigl(1+\frac{1}{x}\Bigr).\label{eq:gprime}
\end{align}

Using the inequality $\log\!\Bigl(1+\frac1x\Bigr)
\ge \frac{1}{x+1}$ (see Lemma~\ref{lem:logineq} for the formal proof), we have
\[
g'(x)\le \frac{3}{4(x+2)}-\frac{1}{x+1}
=\frac{3(x+1)-4(x+2)}{4(x+1)(x+2)}
=-\frac{x+5}{4(x+1)(x+2)}<0.
\]
for all $x>0$.

Therefore, the function $f$ is strictly decreasing on $(0,+\infty)$, so we have
\begin{equation}\label{eq:Asup-le-f300}
A_{\sup}(\alpha_0)\le f(300).
\end{equation}

Direct calculation (see Lemma~\ref{f300} for details) shows that
\begin{equation}\label{eq:f300}
f(300)=\frac{302^{3/4}}{2\cdot 301}\Bigl(\frac{300}{301}\Bigr)^{300}<0.051
\end{equation}

\end{proof}

\begin{lemma}\label{f300}
Let 
\begin{equation}\label{eq:def-fx2}
f(x):=\frac{(x+2)^{3/4}}{2(x+1)}\Bigl(\frac{x}{x+1}\Bigr)^x.
\end{equation}
defined on $x \in (0,+\infty)$. Then we have
\begin{equation}
f(300)=\frac{302^{3/4}}{2\cdot 301}\Bigl(\frac{300}{301}\Bigr)^{300}<0.051
\end{equation}

\end{lemma}

\begin{proof} From \eqref{eq:def-fx2}, we have
\begin{equation}
f(300)=\frac{302^{3/4}}{2\cdot 301}\Bigl(\frac{300}{301}\Bigr)^{300}.
\end{equation}
We bound the two factors in \eqref{eq:f300}.

\smallskip\noindent
{(i) Bound $\left(\frac{300}{301}\right)^{300}$.}
By Lemma~\ref{lem:exp-bounds} applied at $x=-\frac{1}{301}$,
\[
e^{-\frac{1}{301}}\ge 1-\frac{1}{301}=\frac{300}{301}.
\]
Raising both sides to the power $300$ (which preserves the inequality since both sides are positive) yields
\[
\Bigl(\frac{300}{301}\Bigr)^{300}\le e^{-\frac{300}{301}}.
\]
By Lemma~\ref{lem:exp-bounds} (second inequality) applied at $x=\frac{300}{301}\ge 0$,
\[
e^{-\frac{300}{301}}
\le \frac{1}{1+\frac{300}{301}+\frac12\left(\frac{300}{301}\right)^2}
=\frac{90601}{225901}.
\]
Therefore,
\begin{equation}\label{eq:pow-bound}
\Bigl(\frac{300}{301}\Bigr)^{300}\le \frac{90601}{225901}.
\end{equation}

\smallskip\noindent
{(ii) Bound $\frac{302^{3/4}}{301}$.}
Write $302=301\left(1+\frac{1}{301}\right)$, hence
\[
\frac{302^{3/4}}{301}
=301^{-1/4}\Bigl(1+\frac{1}{301}\Bigr)^{3/4}.
\]
Since $301\ge 256=4^4$, we have $301^{-1/4}\le \frac14$.
By Lemma~\ref{lem:bernoulli} with $\alpha=\frac34$ and $u=\frac{1}{301}$,
\[
\Bigl(1+\frac{1}{301}\Bigr)^{3/4}\le 1+\frac{3}{4\cdot 301}=\frac{1207}{1204}.
\]
Therefore
\begin{equation}\label{eq:prefactor-bound}
\frac{302^{3/4}}{301}\le \frac14\cdot\frac{1207}{1204}=\frac{1207}{4816}.
\end{equation}

\smallskip\noindent
{(iii) Combine \eqref{eq:f300}, \eqref{eq:pow-bound}, and \eqref{eq:prefactor-bound}.}
Using \eqref{eq:pow-bound} and \eqref{eq:prefactor-bound} in \eqref{eq:f300} gives
\[
f(300)\le \frac12\cdot \frac{1207}{4816}\cdot \frac{90601}{225901}
=\frac{1207\cdot 90601}{2\cdot 4816\cdot 225901}.
\]
Compute the numerator and denominator:
\[
1207\cdot 90601=109355407,
\qquad
2\cdot 4816\cdot 225901=2175878432.
\]
Since $109355407=301\cdot 363307$ and $2175878432=301\cdot 7228832$, we simplify to
\begin{equation}\label{eq:f300-rational}
f(300)\le \frac{363307}{7228832}.
\end{equation}
Finally, to verify $\frac{363307}{7228832}<0.051=\frac{51}{1000}$, cross-multiply:
\[
363307\cdot 1000=363307000
<368670432=7228832\cdot 51.
\]
Thus $\frac{363307}{7228832}<0.051$.

Therefore, we have
\[
f(300)\le \frac{363307}{7228832}<0.051,
\]
as claimed.
\end{proof}

\section{Properties of $L_\alpha(\theta)$ and connection with $B_{\sup}(\alpha)$}\label{appen: propL}
In this section, we describe the discrete-to-continuous reduction, and the resulting ODE, that is used later to bound $B_{t}(\alpha)$, the second of the two terms in the analysis of Lemma \ref{l:main-technical}.

\begin{remark}[Motivation of the definition of $L_\alpha(\theta)$]\label{rem:motL}
    Write $r:=1-2\rho\in[0,1)$ and $\delta:=-\log r=-\log(1-2\rho)\ge 0$, so that $r=e^{-\delta}$.
Then for $0<\rho<\tfrac12$ we have $\delta>0$ and
\[
(1-2\rho)^{t-i}=r^{\,t-i}=e^{-\delta(t-i)}.
\]
In particular, when $\rho$ is small one has $\delta\sim 2\rho$, so $\delta(t-i)\approx 2\rho(t-i)$.
Now rewrite
\[
B_t(\rho,\alpha)=\rho\sum_{i=1}^t \frac{(1-2\rho)^{t-i}}{i^\alpha}
=\rho\sum_{i=1}^t \frac{e^{-\delta(t-i)}}{i^\alpha}.
\]
Therefore the quantity $t^\alpha B_t(\rho,\alpha)$ 
\[
\rho\,t\int_0^1 e^{-\delta t(1-u)}u^{-\alpha}\,du,
\]
which depends on $t$ and $\rho$ only through the combined parameter $\delta t$.
Therefore, by change of variable
\[
\theta:=\frac{\delta t}{2}=\frac{t}{2}\bigl(-\log(1-2\rho)\bigr),
\]
the resulting integral for comparison takes the form
\[
L_\alpha(\theta)=\theta\int_0^1 e^{-2\theta u}(1-u)^{-\alpha}\,du.
\]
In this way the integral comparison reduces the two-parameter family $(t,\rho)$ to the one-parameter family $\theta$.
\end{remark}

\begin{remark}[A simplification of $L_{3/4}$]\label{rem:L-34}
For every $\theta>0$, the function $L_{3/4}$ admits the non-singular representation
\begin{equation}\label{eq:L-34-nonsing}
L_{3/4}(\theta)
=4\theta e^{-2\theta}\int_0^1 e^{2\theta y^4}\,dy
=4\theta\int_0^1 e^{-2\theta(1-y^4)}\,dy.
\end{equation}
In particular, for each fixed $\theta>0$ the integrand $y\mapsto e^{-2\theta(1-y^4)}$ is $C^\infty$ on $[0,1]$.
\end{remark}

\begin{proof}
Fix $\theta>0$. By definition,
\[
L_{3/4}(\theta)=\theta\int_0^1 e^{-2\theta u}(1-u)^{-3/4}\,du.
\]
Since $\int_0^1 (1-u)^{-3/4}\,du=4<\infty$ and $0\le e^{-2\theta u}\le 1$, the integral is absolutely convergent
(and is understood as an improper integral at $u=1$).

First, for $\varepsilon\in(0,1)$ set
\[
I_\varepsilon:=\int_0^{1-\varepsilon} e^{-2\theta u}(1-u)^{-3/4}\,du.
\]
On $[0,1-\varepsilon]$ the substitution $w=1-u$ is legitimate (it is $C^1$ and strictly monotone), giving
\[
I_\varepsilon=\int_{\varepsilon}^{1} e^{-2\theta(1-w)}w^{-3/4}\,dw
=e^{-2\theta}\int_{\varepsilon}^{1} e^{2\theta w}w^{-3/4}\,dw.
\]
Letting $\varepsilon\to 0^+$ (which is valid since the original integral converges absolutely) yields
\begin{equation}\label{eq:L34-step1}
L_{3/4}(\theta)=\theta e^{-2\theta}\int_0^1 e^{2\theta w}w^{-3/4}\,dw.
\end{equation}

Second, for $\varepsilon\in(0,1)$ define
\[
J_\varepsilon:=\int_{\varepsilon}^{1} e^{2\theta w}w^{-3/4}\,dw.
\]
On $[\varepsilon,1]$ the substitution $w=y^4$ is legitimate (again $C^1$ and strictly increasing), and we obtain
\[
J_\varepsilon
=\int_{\varepsilon^{1/4}}^{1} e^{2\theta y^4}(y^4)^{-3/4}\,4y^3\,dy
=4\int_{\varepsilon^{1/4}}^{1} e^{2\theta y^4}\,dy.
\]
As $\varepsilon\to 0^+$, the left-hand side satisfies $J_\varepsilon\to\int_0^1 e^{2\theta w}w^{-3/4}\,dw$
by the definition of the improper integral.
On the right-hand side, since $0\le e^{2\theta y^4}\le e^{2\theta}$ on $[0,1]$,
\[
0\le 4\int_0^{\varepsilon^{1/4}} e^{2\theta y^4}\,dy \le 4e^{2\theta}\varepsilon^{1/4}\xrightarrow[\varepsilon\to 0^+]{}0,
\]
so
\[
4\int_{\varepsilon^{1/4}}^{1} e^{2\theta y^4}\,dy \xrightarrow[\varepsilon\to 0^+]{} 4\int_{0}^{1} e^{2\theta y^4}\,dy.
\]
Therefore,
\[
\int_0^1 e^{2\theta w}w^{-3/4}\,dw = 4\int_0^1 e^{2\theta y^4}\,dy.
\]
Insert this into \eqref{eq:L34-step1} to obtain
\[
L_{3/4}(\theta)=4\theta e^{-2\theta}\int_0^1 e^{2\theta y^4}\,dy,
\]
which is the first equality in \eqref{eq:L-34-nonsing}. The second equality follows from
$e^{-2\theta}e^{2\theta y^4}=e^{-2\theta(1-y^4)}$.

Finally, since $y\mapsto 1-y^4$ is a polynomial and $x\mapsto e^{x}$ is $C^\infty$ on $\R$,
the composition $y\mapsto e^{-2\theta(1-y^4)}$ is $C^\infty$ on $[0,1]$ for each fixed $\theta>0$.
\end{proof}

\begin{lemma}[Discrete to continuous]\label{lem:Bt-dom}
Fix $0<\alpha<1$, $t\ge1$, and $\rho\in(0,\tfrac12)$.
Then
\begin{equation}\label{eq:Bt-dom}
t^\alpha B_t(\rho,\alpha)\le L_\alpha\Bigl(\frac{t}{2}\bigl(-\log(1-2\rho)\bigr)\Bigr).
\end{equation}
\end{lemma}

\begin{proof}
Set $r:=1-2\rho\in(0,1)$ and $\delta:=-\log r>0$, so $r=e^{-\delta}$.
Reindex with $k=t-i$:
\[
B_t(\rho,\alpha)=\rho\sum_{k=0}^{t-1}\frac{r^k}{(t-k)^\alpha}.
\]
Using $\rho=\frac{1-r}{2}$, for each $k\ge0$, we have
\[
\rho r^k=\frac{1-r}{2}r^k=\frac12(r^k-r^{k+1}).
\]
Since $r=e^{-\delta}$, we have
\[
r^k-r^{k+1}=e^{-\delta k}-e^{-\delta(k+1)}
=\int_k^{k+1}\delta e^{-\delta x}\,dx.
\]
Therefore, we have
\[
\rho r^k=\frac12\int_k^{k+1}\delta e^{-\delta x}\,dx.
\]
Substitute the above identity into the sum, we have
\[
B_t(\rho,\alpha)
=\frac12\sum_{k=0}^{t-1}\int_k^{k+1}\delta e^{-\delta x}\,\frac{dx}{(t-k)^\alpha}.
\]
For $x\in[k,k+1]$, we have $x\ge k$ so $t-x\le t-k$, and since $y\mapsto y^{-\alpha}$ is decreasing on $(0,\infty)$,
\[
(t-k)^{-\alpha}\le (t-x)^{-\alpha}.
\]
Therefore
\[
B_t(\rho,\alpha)\le \frac12\int_0^t \delta e^{-\delta x}(t-x)^{-\alpha}\,dx.
\]
Using change of variable $x=tu$, we have
\[
B_t(\rho,\alpha)\le \frac12\,\delta t^{1-\alpha}\int_0^1 e^{-\delta t u}(1-u)^{-\alpha}\,du.
\]
Multiplying both sides by $t^\alpha$ and setting $\theta=\delta t/2$ shows the desired inequality \eqref{eq:Bt-dom}.
\end{proof}

\begin{lemma}[Endpoint $\rho=\tfrac12$]\label{lem:Bt-endpoint}
Fix $0<\alpha<1$ and $t\ge1$. Then
\[
B_t\Bigl(\tfrac12,\alpha\Bigr)=\frac{1}{2t^\alpha},
\qquad\text{equivalently}\qquad
t^\alpha B_t\Bigl(\tfrac12,\alpha\Bigr)=\frac12.
\]
\end{lemma}

\begin{proof}
If $\rho=\tfrac12$, then $r=1-2\rho=0$.
In $B_t(\rho,\alpha)$, the factor $r^{t-i}$ vanishes unless $i=t$, in which case $r^{0}=1$.
Thus $B_t(\tfrac12,\alpha)=\tfrac12\cdot t^{-\alpha}$.
\end{proof}

\begin{lemma}[ODE for $L_\alpha$]\label{prop:ODE}
Fix $0<\alpha<1$. The function $L_\alpha$ is continuously differentiable on $(0,\infty)$ and satisfies
\begin{equation}\label{eq:ODE2}
L_\alpha'(\theta)=1-\Bigl(2-\frac{\alpha}{\theta}\Bigr)L_\alpha(\theta),
\qquad \theta>0.
\end{equation}
\end{lemma}

\begin{proof}
Define
\[
I(\theta):=\int_0^1 e^{-2\theta u}(1-u)^{-\alpha}\,du,
\qquad\text{so that}\qquad
L_\alpha(\theta)=\theta I(\theta).
\]
For fixed $u\in[0,1)$,
\[
\frac{\partial}{\partial\theta}\bigl(e^{-2\theta u}(1-u)^{-\alpha}\bigr)
=(-2u)e^{-2\theta u}(1-u)^{-\alpha}.
\]
Since $0\le u\le 1$ and $e^{-2\theta u}\le 1$,
\[
\bigl|(-2u)e^{-2\theta u}(1-u)^{-\alpha}\bigr|\le 2(1-u)^{-\alpha}.
\]
Because $0<\alpha<1$, $(1-u)^{-\alpha}\in L^1(0,1)$, hence by differentiation under integration (justified by dominate convergence)
\[
I'(\theta)=\int_0^1 (-2u)e^{-2\theta u}(1-u)^{-\alpha}\,du.
\]
Therefore $L_\alpha'(\theta)=I(\theta)+\theta I'(\theta)$ exists.

Now define
\[
J(\theta):=\int_0^1 e^{-2\theta u}(1-u)^{1-\alpha}\,du.
\]
Since $u(1-u)^{-\alpha}=(1-u)^{-\alpha}-(1-u)^{1-\alpha}$, we have
\[
L_\alpha'(\theta)=(1-2\theta)I(\theta)+2\theta J(\theta).
\]
Let $w(u)=(1-u)^{1-\alpha}$, so $w(0)=1$, $w(1)=0$, and $w'(u)=-(1-\alpha)(1-u)^{-\alpha}$.
Integrating by parts on $[0,1-\varepsilon]$ and letting $\varepsilon\to 0^+$ (justified by Dominated Convergence Theorem since the integrands are dominated by an $L^1$ function),
one obtains
\[
I(\theta)=\frac{1}{1-\alpha}\bigl(1-2\theta J(\theta)\bigr)
\qquad\Longleftrightarrow\qquad
J(\theta)=\frac{1-(1-\alpha)I(\theta)}{2\theta}.
\]
Substituting into $L_\alpha'(\theta)=(1-2\theta)I(\theta)+2\theta J(\theta)$ yields
\[
L_\alpha'(\theta)=1-(2\theta-\alpha)I(\theta)=1-\Bigl(2-\frac{\alpha}{\theta}\Bigr)L_\alpha(\theta),
\]
which is \eqref{eq:ODE2}. Continuity of $L_\alpha'$ follows from continuity of $I$ and $I'$ obtained via dominated convergence.
\end{proof}

\begin{lemma}[Restated Lemma~\ref{lem:onepoint-1}]\label{lem:onepoint}
For any $0<\alpha<1$ and $\ell\in(\tfrac12,1)$, we define
\begin{equation}\label{eq:theta-ell}
\theta_\ell:=\frac{\alpha}{\,2-\frac1\ell\,}.
\end{equation}
If $L_\alpha(\theta_\ell)<\ell$, then $L_\alpha(\theta)<\ell$ for all $\theta>0$, and therefore $\sup_{\theta>0}L_\alpha(\theta)\le \ell$.
\end{lemma}

\begin{proof}
Assume for contradiction that there exists $\theta_0>0$ with $L_\alpha(\theta_0)\ge \ell$.
Since $0<e^{-2\theta u}\le 1$ and $(1-u)^{-\alpha}\in L^1(0,1)$ for $0<\alpha<1$,
\[
0\le L_\alpha(\theta)
=\theta\int_0^1 e^{-2\theta u}(1-u)^{-\alpha}\,du
\le \theta\int_0^1 (1-u)^{-\alpha}\,du
=\frac{\theta}{1-\alpha} \to 0, \text{ as } \theta \to 0
\]
Therefore, $L_\alpha(\theta)<\ell$ for all sufficiently small $\theta>0$.
By continuity of $L_\alpha$ (from its defining integral), the set
$\{\theta>0:\ L_\alpha(\theta)=\ell\}$ is nonempty, and we may define the first hitting time
\[
\tau:=\inf\{\theta>0:\ L_\alpha(\theta)=\ell\}.
\]
Then $\tau>0$, $L_\alpha(\theta)<\ell$ for all $\theta\in(0,\tau)$, and $L_\alpha(\tau)=\ell$.
By Lemma~\ref{lem:hitting-below}, $L_\alpha'(\tau)\ge 0$.

We have the following key observation about the structure  of $L_\alpha'(\theta)$ when $L_\alpha$ passes the line $y=\ell$.
Define for $\theta>0$ the scalar function
\[
\Psi(\theta):=1-\Bigl(2-\frac{\alpha}{\theta}\Bigr)\ell.
\]
Since $L_\alpha$ satisfies the ODE \eqref{eq:ODE}, at any $\theta>0$ with $L_\alpha(\theta)=\ell$ we have
\begin{equation}\label{eq:slope-at-ell}
L_\alpha'(\theta)=\Psi(\theta).
\end{equation}
A direct calculation shows $\Psi'(\theta)=-\frac{\alpha\ell}{\theta^2}<0$, hence $\Psi$ is strictly decreasing on $(0,\infty)$.
Moreover, by the definition \eqref{eq:theta-ell} of $\theta_\ell$ we have
\[
\Psi(\theta_\ell)=1-\Bigl(2-\frac{\alpha}{\theta_\ell}\Bigr)\ell
=1-\Bigl(2-\Bigl(2-\frac1\ell\Bigr)\Bigr)\ell
=1-\frac{1}{\ell}\ell
=0.
\]
Therefore,
\begin{equation}\label{eq:Psi-sign}
\Psi(\theta)>0 \ \text{ for }\ \theta<\theta_\ell,
\qquad
\Psi(\theta)=0 \ \text{ for }\ \theta=\theta_\ell,
\qquad
\Psi(\theta)<0 \ \text{ for }\ \theta>\theta_\ell.
\end{equation}

Note that we cannot conclude that $\theta_\ell$ is global maximizer because $L_\alpha'(\theta)=\Psi(\theta)$ only holds for $L_\alpha(\theta)=\ell$. However, this observation shows that every time when $L_\alpha$ passes the line $y=\ell$ before $\theta_\ell$, the function must be increasing, so it will not be able to go down and satisfy our assumption $L_\alpha(\theta_\ell)<\ell$. The detailed explanation is given below.

First, we claim that the first hit must occur strictly before $\theta_\ell$.
Since $L_\alpha(\tau)=\ell$, \eqref{eq:slope-at-ell} gives $L_\alpha'(\tau)=\Psi(\tau)$.
Because $L_\alpha'(\tau)\ge 0$, we have $\Psi(\tau)\ge 0$.
By \eqref{eq:Psi-sign}, this implies $\tau\le \theta_\ell$.
If $\tau=\theta_\ell$, then $L_\alpha(\theta_\ell)=\ell$, contradicting the hypothesis $L_\alpha(\theta_\ell)<\ell$.
Hence
\begin{equation}\label{eq:tau-strict}
\tau<\theta_\ell.
\end{equation}
By \eqref{eq:Psi-sign} again, $\Psi(\tau)>0$, so $L_\alpha'(\tau)=\Psi(\tau)>0$.

Next, we show that the first return produces a contradiction.
Since $L_\alpha(\tau)=\ell$ and $L_\alpha'(\tau)>0$, continuity implies
there exists $\varepsilon>0$ such that $L_\alpha'(\theta)>0$ for all $\theta\in(\tau,\tau+\varepsilon)$, and therefore $L_\alpha(\theta)>\ell$ for all $\theta\in(\tau,\tau+\varepsilon)$.
On the other hand, $L_\alpha(\theta_\ell)<\ell$ by hypothesis and $\tau<\theta_\ell$ by \eqref{eq:tau-strict},
so by continuity there exists at least one point in $(\tau,\theta_\ell)$ where $L_\alpha$ returns to the level $\ell$.
Define the first return time
\[
\sigma:=\inf\{\theta>\tau:\ L_\alpha(\theta)=\ell\}.
\]
Then $\sigma\in(\tau,\theta_\ell)$, $L_\alpha(\sigma)=\ell$, and $L_\alpha(\theta)>\ell$ for all $\theta\in(\tau,\sigma)$.
By Lemma~\ref{lem:hitting-above}, $L_\alpha'(\sigma)\le 0$.

But since $L_\alpha(\sigma)=\ell$, we have $L_\alpha'(\sigma)=\Psi(\sigma)$ by \eqref{eq:slope-at-ell}.
Because $\sigma<\theta_\ell$, \eqref{eq:Psi-sign} gives $\Psi(\sigma)>0$, hence $L_\alpha'(\sigma)>0$.
This contradicts $L_\alpha'(\sigma)\le 0$.

\medskip
Therefore our assumption was false, and $L_\alpha(\theta)<\ell$ for all $\theta>0$, i.e., $\sup_{\theta>0}L_\alpha(\theta)\le \ell$.
\end{proof}

\begin{lemma}[Series representation for $L_\alpha$]\label{lem:series}
For $0<\alpha<1$ and $\theta>0$,
\begin{equation}\label{eq:series}
L_\alpha(\theta)=\theta e^{-2\theta}\sum_{n=0}^\infty \frac{(2\theta)^n}{n!\,(n+1-\alpha)}.
\end{equation}
\end{lemma}

\begin{proof}
With $w=1-u$, \eqref{eq:L} becomes
\[
L_\alpha(\theta)=\theta e^{-2\theta}\int_0^1 e^{2\theta w}w^{-\alpha}\,dw.
\]
Expand $e^{2\theta w}=\sum_{n=0}^\infty \frac{(2\theta w)^n}{n!}$.
For each $N$, the partial sums are nonnegative and increase pointwise to $e^{2\theta w}$, hence
\[
e^{2\theta w}w^{-\alpha}=\lim_{N\to\infty}\sum_{n=0}^N \frac{(2\theta)^n}{n!}\,w^{n-\alpha}
\quad\text{by monotone convergence.}
\]
Apply monotone convergence to justify termwise integration:
\[
\int_0^1 e^{2\theta w}w^{-\alpha}\,dw
=\sum_{n=0}^\infty \frac{(2\theta)^n}{n!}\int_0^1 w^{n-\alpha}\,dw
=\sum_{n=0}^\infty \frac{(2\theta)^n}{n!\,(n+1-\alpha)}.
\]
Multiplying by $\theta e^{-2\theta}$ gives \eqref{eq:series}.
\end{proof}

\begin{corollary}[A particularly clean series at $\alpha=\tfrac34$]\label{cor:series-34}
Let $\alpha_0=\tfrac34$. For $\theta>0$,
\begin{equation}\label{eq:series-34}
L_{\alpha_0}(\theta)=4\theta e^{-2\theta}\sum_{n=0}^\infty \frac{(2\theta)^n}{n!\,(4n+1)}.
\end{equation}
\end{corollary}

\begin{proof}
Starting from \eqref{eq:L-34-nonsing},
\[
L_{\alpha_0}(\theta)=4\theta e^{-2\theta}\int_0^1 e^{2\theta y^4}\,dy.
\]
Expand $e^{2\theta y^4}=\sum_{n=0}^\infty \frac{(2\theta)^n}{n!}y^{4n}$, and use monotone convergence (all terms nonnegative) to integrate termwise:
\[
\int_0^1 e^{2\theta y^4}\,dy
=\sum_{n=0}^\infty \frac{(2\theta)^n}{n!}\int_0^1 y^{4n}\,dy
=\sum_{n=0}^\infty \frac{(2\theta)^n}{n!\,(4n+1)}.
\]
Multiplying by $4\theta e^{-2\theta}$ gives \eqref{eq:series-34}.
\end{proof}

\begin{lemma}[Monotone ratios for the series coefficients]\label{lem:ratio-monotone}
		Fix $\alpha\in(0,1)$ and $x>0$, and define
		\[
		a_n:=\frac{x^n}{n!\,(n+1-\alpha)}\qquad(n\ge0).
		\]
		Then the ratio $r_n:=a_{n+1}/a_n$ is strictly decreasing for all integers $n\ge 1$.
	\end{lemma}
	
	\begin{proof}
		A direct computation gives
		\[
		r_n=\frac{a_{n+1}}{a_n}
		=\frac{x}{n+1}\cdot\frac{n+1-\alpha}{n+2-\alpha}.
		\]
		For $n\ge1$,
		\[
		\frac{r_{n+1}}{r_n}
		=
		\frac{n+1}{n+2}\cdot
		\frac{(n+2-\alpha)^2}{(n+1-\alpha)(n+3-\alpha)}.
		\]
		Thus $r_{n+1}\le r_n$ is equivalent to
		\[
		(n+2)(n+1-\alpha)(n+3-\alpha)-(n+1)(n+2-\alpha)^2\ge 0.
		\]
		Expanding the left-hand side gives
		\[
		n^2+(3-2\alpha)n+(2-4\alpha+\alpha^2).
		\]
		For $\alpha\in(0,1)$ this quadratic is strictly increasing for $n\ge0$ and positive at $n=1$:
		\[
		1+(3-2\alpha)+(2-4\alpha+\alpha^2)=\alpha^2-6\alpha+6\ge 1.
		\]
		Hence it is positive for all $n\ge1$, and therefore $r_{n+1}<r_n$ for all $n\ge1$.
	\end{proof}

\section{Exact Calculations for $L_{3/4}$ and $B_{\sup}(3/4)$}\label{app:cert}

In this section we present the final calculations needed to complete the proof of the $\alpha=3/4$ version of Lemma \ref{l:main-technical}.
\begin{lemma}[One-point bound for $L_{\alpha_0}$]\label{lem:L-onepoint}
With $\alpha_0=\frac34$ and $\ell_0=\frac{497}{500}$,
\[
L_{\alpha_0}\Bigl(\theta_{\ell_0}\Bigr)<\ell_0.
\]
\end{lemma}

\begin{proof}

\noindent{(i) Rewrite $L_{\alpha_0}(\theta_{\ell_0})$ as a positive series}

Let
\[
\theta:=\frac{1491}{1976},
\qquad
x:=2\theta=\frac{1491}{988}.
\]
By Corollary~\ref{cor:series-34},
\begin{equation}\label{eq:Lseries-34}
L_{\alpha_0}(\theta)
=4\theta e^{-x}\sum_{n=0}^\infty \frac{x^n}{n!\,(4n+1)}.
\end{equation}
Define the positive series
\[
S(x):=\sum_{n=0}^\infty \frac{x^n}{n!\,(4n+1)}.
\]

\noindent{(ii) An explicit rational upper bound for $e^{-x}$}

Since $e^x=\sum_{k=0}^\infty x^k/k!$ has positive terms for $x\ge0$, for any integer $m\ge0$,
\[
e^x\ge P_m(x):=\sum_{k=0}^m \frac{x^k}{k!}.
\]
Therefore
\begin{equation}\label{eq:eminus}
e^{-x}\le \frac{1}{P_m(x)}.
\end{equation}

\noindent{(iii) Rational truncation and tail bound for $S(x)$}

Fix $N\ge0$ and write
\[
S_N(x):=\sum_{n=0}^N \frac{x^n}{n!\,(4n+1)}.
\]
Since $n\mapsto (4n+1)^{-1}$ is decreasing, for $n\ge N+1$,
\[
\frac{1}{4n+1}\le \frac{1}{4(N+1)+1}=\frac{1}{4N+5}.
\]
Hence
\[
S(x)\le S_N(x)+\frac{1}{4N+5}\sum_{n=N+1}^\infty \frac{x^n}{n!}.
\]
Assume $x<N+2$. For $n\ge N+1$,
\[
\frac{x^{n+1}/(n+1)!}{x^n/n!}=\frac{x}{n+1}\le \frac{x}{N+2}<1,
\]
so the exponential tail is bounded by a geometric series:
\begin{equation}\label{eq:tail}
\sum_{n=N+1}^\infty \frac{x^n}{n!}
\le \frac{x^{N+1}}{(N+1)!}\cdot\frac{1}{1-\frac{x}{N+2}}.
\end{equation}
Combining,
\begin{equation}\label{eq:Supper}
S(x)\le S_N(x)+\frac{1}{4N+5}\cdot \frac{x^{N+1}}{(N+1)!}\cdot\frac{1}{1-\frac{x}{N+2}}.
\end{equation}

\noindent{(iv) Choose $(N,m)$ and conclude}

Combining \eqref{eq:Lseries-34}, \eqref{eq:eminus}, and \eqref{eq:Supper} yields
\[
L_{\alpha_0}(\theta)
\le
4\theta\cdot \frac{1}{P_m(x)}\cdot
\Biggl[
S_N(x)+\frac{1}{4N+5}\cdot \frac{x^{N+1}}{(N+1)!}\cdot\frac{1}{1-\frac{x}{N+2}}
\Biggr].
\]
Take $N=4$ and $m=8$. Since $x=\frac{1491}{988}<6=N+2$, the condition $x<N+2$ holds.

Evaluating the right-hand side in exact rational arithmetic gives
\[
Q_{4,8}:=
4\theta\cdot \frac{1}{P_{8}(x)}\cdot
\Biggl[
S_{4}(x)+\frac{1}{21}\cdot \frac{x^{5}}{5!}\cdot\frac{1}{1-\frac{x}{6}}
\Biggr]
=
\frac{1287675113562193776446577567744}{1295590272667287121985809656211}.
\]
Moreover,
\[
\ell_0-Q_{4,8}
=
\frac{70808734544811403658615264867}{647795136333643560992904828105500}
>0.
\]
Hence $Q_{4,8}<\ell_0$, and therefore
\[
L_{\alpha_0}(\theta_{\ell_0})\le Q_{4,8}<\ell_0,
\]
proving Lemma~\ref{lem:L-onepoint}.
\end{proof}

\begin{lemma}[Bound for $B_{\sup}$]\label{lem:large-t}
With $\alpha_0=\frac34$, we have $B_{\sup}(\alpha_0)\le \frac{201}{200}\cdot\frac{497}{500}=\frac{99897}{100000}<1$.
\end{lemma}

\begin{proof}
Fix $t\ge 300$ and $\rho\in(0,\tfrac12]$. Write
\[
(t+2)^{\alpha_0}B_t
=\Bigl(1+\frac{2}{t}\Bigr)^{\alpha_0} t^{\alpha_0}B_t.
\]
By Lemma~\ref{lem:bernoulli} with $u=\frac{2}{t}$,
\[
\Bigl(1+\frac{2}{t}\Bigr)^{\alpha_0}\le 1+\alpha_0\cdot\frac{2}{t}\le 1+\alpha_0\cdot\frac{2}{300}=1+\frac{2\alpha_0}{300}.
\]
Using Lemma~\ref{lem:L-onepoint} and Lemma~\ref{lem:onepoint}, we have $t^{\alpha_0}B_t\le \ell_0$, hence
\[
(t+2)^{\alpha_0}B_t \le \Bigl(1+\frac{2\alpha_0}{300}\Bigr)\ell_0 =\frac{201}{200}\cdot\frac{497}{500}=\frac{99897}{100000}<1.
\]

\end{proof}

\section{Remaining Details for the Proof of Lemma \ref{l:main-technical} with $\alpha=3/4+0.001$}
\label{app:main-tech-full}

We will show that there exists an absolute constant $C>0$ such that for all $t\geq 1000$, $K\geq C$, $\rho\in(0,1/2]$, and
$\alpha = 3/4+0.001$, we have:
\begin{align*}
	\rho(1-2\rho)^t +
	K\rho\sum_{i=1}^t\frac{(1-2\rho)^{t-i}}{i^\alpha} \leq \frac{K}{(t+2)^\alpha}.
\end{align*}
and then we will show that we can also choose $C=1000$.

In this section, we define
\begin{equation}
A_{\sup}(\alpha):=\sup_{t\ge 1000}\ \sup_{\rho\in(0,\tfrac12]}\ (t+2)^\alpha A_t(\rho),
\qquad
B_{\sup}(\alpha):=\sup_{t\ge 1000}\ \sup_{\rho\in(0,\tfrac12]}\ (t+2)^\alpha B_t(\rho,\alpha).
\end{equation}

The first term $A_{\sup}(\alpha)$ is easier to bound as before.

\begin{lemma}\label{lem:f-decreasing}
	For every $\alpha\in(0,1]$, the function $t\mapsto f_\alpha(t)$ defined as
	\begin{align*}
		f_\alpha(t)=\frac{(t+2)^{\alpha}}{2(t+1)}\Bigl(\frac{t}{t+1}\Bigr)^t
	\end{align*}
	is strictly decreasing for all real $t>0$.
    Consequently, we have $A_{\mathrm{sup}}(\aZero)<\frac{15251}{334150}<0.046$.
\end{lemma}

\begin{proof}
	For real $t>0$ write
	\[
	\log f_\alpha(t)=\alpha\log(t+2)-\log\bigl(2(t+1)\bigr)+t\log\Bigl(\frac{t}{t+1}\Bigr).
	\]
	Differentiating gives
	\[
	\frac{d}{dt}\log f_\alpha(t)=\frac{\alpha}{t+2}-\log\Bigl(1+\frac{1}{t}\Bigr).
	\]
	Using see Lemma~\ref{lem:logineq} for $t>0$, we have
	\[
	\frac{d}{dt}\log f_\alpha(t)
	\le \frac{\alpha}{t+2}-\frac{1}{t+1}
	\le \frac{1}{t+2}-\frac{1}{t+1}<0,
	\]
	since $\alpha\le 1$. Therefore $\log f_\alpha$ is strictly decreasing, and so is $f_\alpha$.
    By Lemma~\ref{lem:A-max}, we have
\begin{equation*}
\sup_{\rho\in(0,\tfrac12]}(t+2)^{\alpha_0}A_t(\rho)
=\frac{(t+2)^{\alpha_0}}{2(t+1)}\Bigl(\frac{t}{t+1}\Bigr)^t.
\end{equation*}

Therefore, we have $A_{\mathrm{sup}}(\aZero)<f_{\aZero}(1000)<f_{\aZero}(300)$. By Lemma~\ref{lem:f300}, we have
\begin{align*}
    A_{\mathrm{sup}}(\aZero)<f_{\aZero}(300)<\frac{15251}{334150}<0.046
\end{align*}

\end{proof}

\begin{lemma}\label{lem:f300}
	With $\aZero=\frac{751}{1000}$,
	\[
	f_{\aZero}(300)\;<\;\frac{15251}{334150}\;<\;0.046.
	\]
\end{lemma}

\begin{proof}
	By Lemma~\ref{lem:A-max},
	\[
	f_{\aZero}(300)=\frac{302^{751/1000}}{2\cdot 301}\Bigl(\frac{300}{301}\Bigr)^{300}.
	\]
	Since $751/1000=1-249/1000$,
	\[
	302^{751/1000}=302\cdot 302^{-249/1000}\le 302\cdot 301^{-249/1000},
	\]
	because $x\mapsto x^{-249/1000}$ is decreasing and $302>301$. Hence
	\begin{equation}\label{eq:f300-1}
		f_{\aZero}(300)\le \frac{151}{301}\cdot 301^{-249/1000}\cdot \Bigl(\frac{300}{301}\Bigr)^{300}.
	\end{equation}
	
	\smallskip\noindent
	{(i) Bounding $301^{-249/1000}$.}
	Write $301^{249/1000}=301^{1/4-1/1000}=301^{1/4}/301^{1/1000}$.
	We bound numerator and denominator separately.
	Since $(41/10)^4=\frac{2825761}{10000}<301$, we have $301^{1/4}>41/10$.
	Also,
	\[
	\Bigl(\frac{101}{100}\Bigr)^{100}
	=\sum_{k=0}^{100}\binom{100}{k}\frac{1}{100^k}
	>1+\binom{100}{1}\frac{1}{100}=2,
	\]
	so $(101/100)^{1000}>2^{10}=1024>301$, hence $301^{1/1000}<101/100$.
	Therefore
	\[
	301^{249/1000}=\frac{301^{1/4}}{301^{1/1000}}
	>\frac{41/10}{101/100}=\frac{410}{101},
	\qquad\text{so}\qquad
	301^{-249/1000}<\frac{101}{410}.
	\]
	
	\smallskip\noindent
	{(ii) Bounding $(300/301)^{300}$.}
	The standard inequality $(1-\frac1n)^n<e^{-1}$ (for $n\ge 1$) gives
	\[
	\Bigl(\frac{300}{301}\Bigr)^{301}<\frac{1}{e}.
	\]
	Thus
	\[
	\Bigl(\frac{300}{301}\Bigr)^{300}
	=\frac{301}{300}\Bigl(\frac{300}{301}\Bigr)^{301}
	<\frac{301}{300}\cdot \frac{1}{e}.
	\]
	Using the $6$-term lower bound $e>1+1+\frac12+\frac16+\frac1{24}+\frac1{120}=\frac{163}{60}$ yields $1/e<60/163$, hence
	\[
	\Bigl(\frac{300}{301}\Bigr)^{300}<\frac{301}{300}\cdot \frac{60}{163}=\frac{301}{815}.
	\]
	
	\smallskip\noindent
	{(iii) Conclusion.}
	Insert the bounds from (i)--(ii) into \eqref{eq:f300-1}:
	\[
	f_{\aZero}(300)
	<\frac{151}{301}\cdot \frac{101}{410}\cdot \frac{301}{815}
	=\frac{151\cdot 101}{410\cdot 815}
	=\frac{15251}{334150}
	<0.046.
	\]
	Finally, since $\tZero=1000>300$, Lemma~\ref{lem:f-decreasing} implies
	$A_{\mathrm{sup}}(\aZero)=\sup_{t\ge 1000} f_{\aZero}(t)\le f_{\aZero}(300)$.
\end{proof}

To bound the second term, we compare it with the function
\begin{equation}
	L_\alpha(\theta):=\theta\int_0^1 e^{-2\theta u}(1-u)^{-\alpha}\,du.
\end{equation}
as before.

\begin{lemma}[One-point inequality at $\ell=\frac{499}{500}$]\label{lem:onepoint-check}
	Let $\aZero=\frac{751}{1000}$ and $\ell=\frac{499}{500}$, with
	\begin{equation}\label{eq:thetaell-value}
		\theta_\ell=\frac{\aZero}{2-\frac{1}{\ell}}
		=\frac{751/1000}{2-500/499}
		=\frac{374749}{498000}.
	\end{equation}
	Then $\Ls_{\aZero}(\theta_\ell)<\ell$.
\end{lemma}

\begin{proof}
	Set $x:=2\theta_\ell$. By \eqref{eq:thetaell-value},
	\[
	x=\frac{374749}{249000}=\frac{301}{200}+\frac{1}{62250}.
	\]
	In particular $x>\frac{301}{200}$.
	Also $x<\frac{753}{500}$ since $\frac{753}{500}-\frac{374749}{249000}=\frac{49}{49800}>0$.
	Let
	\[
	x_+:=\frac{753}{500},\qquad \theta_+:=\frac{x_+}{2}=\frac{753}{1000}.
	\]
	Then $\theta_\ell<\theta_+$, $e^{-x}<e^{-301/200}$, and by \eqref{eq:series},
	\begin{equation}\label{eq:Lbound-start}
		\Ls_{\aZero}(\theta_\ell)
		=\theta_\ell e^{-x}\sum_{n=0}^{\infty}\frac{x^n}{n!\,(n+1-\aZero)}
		<\theta_+\,e^{-301/200}\sum_{n=0}^{\infty}\frac{x_+^n}{n!\,(n+1-\aZero)}.
	\end{equation}
	Define
	\[
	a_n:=\frac{x_+^n}{n!\,(n+1-\aZero)}\qquad(n\ge0),
	\qquad
	S(x_+):=\sum_{n=0}^{\infty}a_n.
	\]
	
	\smallskip\noindent
	{(i) A rational bound on $S(x_+)$.}
	A direct calculation gives the first five terms
	\[
	a_0=\frac{1000}{249},\quad
	a_1=\frac{1506}{1249},\quad
	a_2=\frac{567009}{1124500},\quad
	a_3=\frac{15813251}{90250000},\quad
	a_4=\frac{107166402027}{2124500000000}.
	\]
	Moreover,
	\[
	r_4=\frac{a_{5}}{a_4}
	=\frac{x_+}{5}\cdot \frac{5-\aZero}{6-\aZero}
	=\frac{3199497}{13122500}
	<\frac14,
	\]
	since $4\cdot 3199497=12797988<13122500$.
	By Lemma~\ref{lem:ratio-monotone}, $r_n\le r_4$ for all $n\ge4$, hence
	\[
	\sum_{n\ge4}a_n \le a_4\sum_{j\ge0}r_4^j \le a_4\sum_{j\ge0}\Bigl(\frac14\Bigr)^j=\frac{4}{3}a_4.
	\]
	Therefore
	\begin{equation}\label{eq:Sxplus}
		S(x_+)\le a_0+a_1+a_2+a_3+\frac{4}{3}a_4
		=
		\frac{800429153543344037119501}{134108154748420125000000}.
	\end{equation}

    \smallskip\noindent
	{(ii) A rational bound on $e^{-301/200}$.}
	Let $P_8(y):=\sum_{j=0}^{8}\frac{y^j}{j!}$.
	Since $e^{y}=P_8(y)+\sum_{j\ge 9}\frac{y^j}{j!}>P_8(y)$ for $y>0$, we have $e^{-y}<1/P_8(y)$.
	For $y=\frac{301}{200}$ one computes
	\begin{equation}\label{eq:P8}
		P_8\!\Bigl(\frac{301}{200}\Bigr)
		=
		\frac{66414558043759180589143}{14745600000000000000000},
		\qquad\text{so}\qquad
		e^{-301/200}<\frac{14745600000000000000000}{66414558043759180589143}.
	\end{equation}
	
	\smallskip\noindent
	{(iii) Conclude $\Ls_{\aZero}(\theta_\ell)<\ell$.}
	Combining \eqref{eq:Lbound-start}, \eqref{eq:Sxplus}, and \eqref{eq:P8} yields the explicit bound
	\begin{align*}
		&\Ls_{\aZero}(\theta_\ell)
		\\<&
		\frac{753}{1000}\cdot
		\frac{14745600000000000000000}{66414558043759180589143}
		\cdot
		\frac{800429153543344037119501}{134108154748420125000000}
		\\=&
		\frac{23700038717989377538168622402764800000}
		{23751290207147698032780270875855940541}.
	\end{align*}
	
	To show this is $<\ell=\frac{499}{500}$, it suffices to check the strict integer inequality
	\begin{align*}
		&499\cdot 23751290207147698032780270875855940541
		\\&-
		500\cdot 23700038717989377538168622402764800000
		\\=&
		1874454372012549273043965669714329959
		>0,
	\end{align*}
	which completes the proof.
\end{proof}

\begin{lemma}[A bound on $B_{\mathrm{sup}}(\aZero)$]\label{lem:Bsup}
	With $\aZero=\frac{751}{1000}$ and $\ell=\frac{499}{500}$,
	\[
	B_{\mathrm{sup}}(\aZero)\le \frac{249874749}{250000000}
	\qquad\text{and hence}\qquad
	1-B_{\mathrm{sup}}(\aZero)\ge \frac{125251}{250000000}.
	\]
\end{lemma}

\begin{proof} 	Fix $t\ge 1000$ and $\rho\in(0,\tfrac12]$.
	If $\rho\in(0,\tfrac12)$, Lemma~\ref{lem:Bt-dom} and Lemma~\ref{lem:onepoint} give
	$t^{\aZero}B_t(\rho,\aZero)\le \ell$.
	If $\rho=\tfrac12$, Lemma~\ref{lem:Bt-endpoint} gives
	$t^{\aZero}B_t(\tfrac12,\aZero)=\tfrac12\le \ell$.
	
	Thus, for all $\rho\in(0,\tfrac12]$,
	\[
	(t+2)^{\aZero}B_t(\rho,\aZero)
	=\Bigl(1+\frac{2}{t}\Bigr)^{\aZero} t^{\aZero}B_t(\rho,\aZero)
	\le \Bigl(1+\frac{2}{t}\Bigr)^{\aZero}\ell.
	\]
	Since $\aZero\in(0,1)$, Bernoulli's inequality gives $(1+u)^{\aZero}\le 1+\aZero u$ for $u\ge 0$.
	With $u=2/t$ and $t\ge 1000$,
	\[
	\Bigl(1+\frac{2}{t}\Bigr)^{\aZero}\le 1+\frac{2\aZero}{t}\le 1+\frac{2\aZero}{1000}
	=1+\frac{751}{500000}
	=\frac{500751}{500000}.
	\]
	Therefore
	\[
	(t+2)^{\aZero}B_t(\rho,\aZero)\le \frac{500751}{500000}\cdot \frac{499}{500}
	=\frac{249874749}{250000000}.
	\]
	Taking the supremum over $t\ge 1000$ and $\rho\in(0,\tfrac12]$ yields the stated bound on $B_{\mathrm{sup}}(\aZero)$.
	The bound on $1-B_{\mathrm{sup}}(\aZero)$ is immediate.
\end{proof}

\begin{lemma}[Requirement for $K$]\label{lem:K0}
	With the bounds from Lemma~\ref{lem:f300} and Lemma~\ref{lem:Bsup},
	\[
	\frac{A_{\mathrm{sup}}(\aZero)}{1-B_{\mathrm{sup}}(\aZero)}
	<
	\frac{76255000000}{837052433}
	<100<1000.
	\]
\end{lemma}

\begin{proof}
	Using $A_{\mathrm{sup}}(\aZero)<\frac{15251}{334150}$ and $1-B_{\mathrm{sup}}(\aZero)\ge \frac{125251}{250000000}$,
	\[
	\frac{A_{\mathrm{sup}}(\aZero)}{1-B_{\mathrm{sup}}(\aZero)}
	<
	\frac{\frac{15251}{334150}}{\frac{125251}{250000000}}
	=
	\frac{76255000000}{837052433}.
	\]
	Finally,
	\[
	\frac{76255000000}{837052433}<100
	\iff
	76255000000<100\cdot 837052433=83705243300,
	\]
	which is immediate.
\end{proof}

\section{Lower Bound for $\|\Nb_t\|$}
\label{app:lower-bound}

In this section, we give an upper bound for the parameter $\alpha$ that is admissible to make the inequality in Lemma \ref{l:main-technical} hold and then give a lower bound for $\|\Nb_t\|$.

\begin{theorem}[Upper bound on the admissible exponent]\label{thm:alpha-upper}
Let
\[
\aZero:=\frac{753}{1000}=0.753.
\]
Fix any $\alpha\ge \aZero$.
Then for every $K>0$ and every integer $t_0\ge 1$, there exist an integer $t\ge t_0$
and a value $\rho\in(0,\tfrac12)$ such that 
\begin{equation}\label{eq:main-ineq}
    \rho(1-2\rho)^t +
    K\rho\sum_{i=1}^t\frac{(1-2\rho)^{t-i}}{i^\alpha} \leq \frac{K}{(t+2)^\alpha}.
\end{equation}
fails, i.e.,
\[
\rho(1-2\rho)^t
\;+
K\rho\sum_{i=1}^t \frac{(1-2\rho)^{t-i}}{i^\alpha}
\;>
\frac{K}{(t+2)^\alpha}.
\]
In particular, if there exist constants $K>0$ and $T\ge 1$ such that \eqref{eq:main-ineq}
holds for all integers $t\ge T$ and all $\rho\in(0,\tfrac12]$, then necessarily
\[
\alpha<\frac{753}{1000}=0.753.
\]
\end{theorem}

To prove Theorem \ref{thm:alpha-upper}, we make the following preparations.

Recall that
\begin{align*}
A_t(\rho):=\rho(1-2\rho)^t \quad\text{and}\quad
B_t(\rho,\alpha):=\rho\sum_{i=1}^t \frac{(1-2\rho)^{t-i}}{i^\alpha}.
\end{align*}
and then the inequality is equivalent to
\begin{align*}
    (t+2)^\alpha A_t(\rho) +
    K(t+2)^\alpha B_t(\rho,\alpha) \leq K.
  \end{align*}
The first reduction is to observe that in order to make the inequality \eqref{eq:main-ineq} fail, it suffices to show that $(t+2)^\alpha B_t(\rho,\alpha)$ is large.

\begin{lemma}[Reducing to the second term]\label{lem:reduce-to-B}
For every $t\in\N$, $\rho\in(0,\tfrac12]$, and $\alpha>0$, one has $A_t(\rho)\ge 0$.
Consequently, if \eqref{eq:main-ineq} holds for some $t\in\N$, $\rho\in(0,\tfrac12]$, $K>0$, and $\alpha>0$, then necessarily
\begin{equation}\label{eq:necessary-scaled}
(t+2)^\alpha B_t(\rho,\alpha)\le 1.
\end{equation}
\end{lemma}

\begin{proof}
Since $\rho\in(0,\tfrac12]$, we have $1-2\rho\in[0,1)$, hence $(1-2\rho)^t\ge 0$ and therefore
$A_t(\rho)=\rho(1-2\rho)^t\ge 0$.
If \eqref{eq:main-ineq} holds, then subtracting the nonnegative term $A_t(\rho)$ from the left-hand side gives
\[
K {(t+2)^\alpha} B_t(\rho,\alpha)\le {K}.
\]
Dividing by $K>0$ gives the desired result.
\end{proof}

The next reduction is to observe that if $\alpha_0$ makes the inequality \eqref{eq:main-ineq} fail, then any $\alpha>\alpha_0$ will also make the inequality fail, because of the monotonicity below.

\begin{lemma}[Monotonicity of $t^\alpha B_t(\rho,\alpha)$ in $\alpha$]\label{lem:B-monotone-alpha}
Fix $t\in\N$ and $\rho\in(0,\tfrac12]$.  Then the map
\[
\alpha\longmapsto t^\alpha B_t(\rho,\alpha)
\]
is nondecreasing on $[0,\infty)$.  Equivalently, if $0\le \alpha_1\le \alpha_2$, then
\begin{equation}\label{eq:B-monotone-alpha}
t^{\alpha_1}B_t(\rho,\alpha_1)\le t^{\alpha_2}B_t(\rho,\alpha_2).
\end{equation}
\end{lemma}

\begin{proof}
Write $r:=1-2\rho\in[0,1)$.  Then
\[
t^\alpha B_t(\rho,\alpha)
=
\rho\sum_{i=1}^t r^{t-i}\Bigl(\frac{t}{i}\Bigr)^\alpha.
\]
For each $i\in\{1,\dots,t\}$ we have $t/i\ge 1$, so the function
$\alpha\mapsto (t/i)^\alpha$ is nondecreasing on $[0,\infty)$.  Every coefficient
$\rho r^{t-i}$ is nonnegative, hence the whole sum is nondecreasing in $\alpha$.
This proves \eqref{eq:B-monotone-alpha}.
\end{proof}

For $\alpha\in(0,1)$ and $\theta>0$, recall that in section~\ref{s:main-technical} we define
\begin{equation}\label{eq:def-L}
L_\alpha(\theta)=\theta\int_0^1 e^{-2\theta u}(1-u)^{-\alpha}\,du.
\end{equation}
to get upper bound for the quantity $t^\alpha B_t(\rho,\alpha)$.

Now to lower bound $t^\alpha B_t(\rho,\alpha)$, for $\alpha\in(0,1)$ and $\theta>0$, we define a similar quantity
\begin{equation}\label{eq:def-Lminus}
\Gamma_{\alpha,t}(\theta)=\theta\int_0^1 e^{-2\theta u}\Bigl(1-u+\frac1t\Bigr)^{-\alpha}\,du.
\end{equation}
Unlike $L_\alpha(\theta)$, the integral $\Gamma_{\alpha,t}(\theta)$ is a proper integral.

\begin{lemma}[Discrete-to-continuous lower bound]\label{lem:Bt-lower}
Fix $\alpha\in(0,1)$, $t\in\N$, and $\rho\in(0,\tfrac12)$.
Then
\begin{equation}\label{eq:Bt-lower-claim}
t^\alpha B_t(\rho,\alpha)\ge \Gamma_{\alpha,t}\Bigl(\frac{t}{2}\bigl(-\log(1-2\rho)\bigr)\Bigr).
\end{equation}
\end{lemma}

\begin{proof}
As in the proof of Lemma \ref{lem:Bt-dom}, let
$
r=1-2\rho\in(0,1),
\delta=-\log r>0,
$
and $\theta=\frac{\delta t}{2}=\frac{t}{2}\bigl(-\log(1-2\rho)\bigr)$, and write $B_t(\rho,\alpha)$ in the integral form
\begin{align}
B_t(\rho,\alpha)
&=\frac12\sum_{k=0}^{t-1}\frac{1}{(t-k)^\alpha}\int_k^{k+1}\delta e^{-\delta x}\,dx \notag\\
&=\frac12\sum_{k=0}^{t-1}\int_k^{k+1}\delta e^{-\delta x}(t-k)^{-\alpha}\,dx. \label{eq:B-sum-of-integrals}
\end{align}

Since we want a lower bound for $t^\alpha B_t(\rho,\alpha)$, we want to lower bound the factor $(t-k)^{-\alpha}$ to make it not depend on $k$. Fix $k\in\{0,1,\dots,t-1\}$ and $x\in[k,k+1]$.  Then we have
\[
t-x+1\ge t-(k+1)+1=t-k.
\]
Because $\alpha>0$, the function $y\mapsto y^{-\alpha}$ is decreasing on $(0,\infty)$, so
\[
(t-k)^{-\alpha}\ge (t-x+1)^{-\alpha}.
\]
Multiplying by the nonnegative factor $\delta e^{-\delta x}$ and integrating over $[k,k+1]$ gives
\[
\int_k^{k+1}\delta e^{-\delta x}(t-k)^{-\alpha}\,dx
\ge
\int_k^{k+1}\delta e^{-\delta x}(t-x+1)^{-\alpha}\,dx.
\]
Summing over $k$ and using \eqref{eq:B-sum-of-integrals}, we obtain
\begin{align}
B_t(\rho,\alpha)
&\ge \frac12\sum_{k=0}^{t-1}\int_k^{k+1}\delta e^{-\delta x}(t-x+1)^{-\alpha}\,dx \notag\\
&=\frac12\int_0^t \delta e^{-\delta x}(t-x+1)^{-\alpha}\,dx. \label{eq:B-integral-lower}
\end{align}

Set $x=tu$, so $u\in[0,1]$ and $dx=t\,du$.  Then \eqref{eq:B-integral-lower} becomes
\begin{align}
B_t(\rho,\alpha)
&\ge \frac12\int_0^1 \delta e^{-\delta tu}\bigl(t-tu+1\bigr)^{-\alpha}t\,du \notag\\
&=\frac12\,\delta\,t\int_0^1 e^{-\delta tu}\Bigl(t\Bigl(1-u+\frac1t\Bigr)\Bigr)^{-\alpha}\,du \notag\\
&=\frac12\,\delta\,t^{1-\alpha}\int_0^1 e^{-\delta tu}\Bigl(1-u+\frac1t\Bigr)^{-\alpha}\,du. \label{eq:B-after-rescale}
\end{align}

Multiplying \eqref{eq:B-after-rescale} by $t^\alpha$ gives
\[
t^\alpha B_t(\rho,\alpha)
\ge
\frac12\,\delta\,t\int_0^1 e^{-\delta tu}\Bigl(1-u+\frac1t\Bigr)^{-\alpha}\,du.
\]
Now $\delta t=2\theta$, so $e^{-\delta tu}=e^{-2\theta u}$ and $\frac12\delta t=\theta$.
Hence
\[
t^\alpha B_t(\rho,\alpha)
\ge
\theta\int_0^1 e^{-2\theta u}\Bigl(1-u+\frac1t\Bigr)^{-\alpha}\,du
=
\Gamma_{\alpha,t}(\theta)=\Gamma_{\alpha,t}\Bigl(\frac{t}{2}\bigl(-\log(1-2\rho)\bigr)\Bigr),
\]
which is exactly \eqref{eq:Bt-lower-claim}.
\end{proof}

Although to lower bound $t^\alpha B_t(\rho,\alpha)$ we use $\Gamma_{\alpha,t}(\theta)$ instead of $L_\alpha(\theta)$, the following convergence result allows us to transform from $\Gamma_{\alpha,t}(\theta)$ to $L_\alpha(\theta)$.

\begin{lemma}[Monotone convergence]\label{lem:Lminus-to-L}
Fix $\alpha\in(0,1)$ and $\theta>0$.  Then the sequence $t\mapsto \Gamma_{\alpha,t}(\theta)$
is nondecreasing and
\begin{equation}\label{eq:Lminus-limit}
\lim_{t\to\infty}\Gamma_{\alpha,t}(\theta)=L_\alpha(\theta).
\end{equation}
\end{lemma}

\begin{proof}
Fix $u\in[0,1)$.  As $t$ increases, the quantity $1-u+\frac1t$ decreases.  Since $\alpha>0$,
the function $x\mapsto x^{-\alpha}$ is decreasing on $(0,\infty)$, so $\Bigl(1-u+\frac1t\Bigr)^{-\alpha}$ is increasing in $t$ and
\[
\lim_{t\to\infty}\Bigl(1-u+\frac1t\Bigr)^{-\alpha}
=
(1-u)^{-\alpha}
\]

Multiplying by the nonnegative factor $e^{-2\theta u}$ preserves monotonicity, so by the Monotone
Convergence Theorem, we have
\[
\lim_{t\to\infty}\int_0^1 e^{-2\theta u}\Bigl(1-u+\frac1t\Bigr)^{-\alpha}\,du
=
\int_0^1 e^{-2\theta u}(1-u)^{-\alpha}\,du.
\]
Multiplying by $\theta>0$ gives \eqref{eq:Lminus-limit}.
\end{proof}

Combining above reductions, to show that there exist an integer $t\ge t_0$
and a value $\rho\in(0,\tfrac12)$ such that \eqref{eq:main-ineq}
fails with $\alpha=\alpha_0$, we only need to show $L_{\alpha_0}(\theta_0)>1$ for some properly chosen $\theta_0$. To this end, we choose $\theta_0=\frac34$ and use the series approximation
\begin{equation*}
L_\alpha(\theta)=\theta e^{-2\theta}\sum_{n=0}^\infty \frac{(2\theta)^n}{n!\,(n+1-\alpha)}.
\end{equation*}
from Lemma \ref{lem:series} to give explicit lower bound for $L_{\aZero}(\theta_0)$.

\begin{lemma}[Lower bound for L]\label{lem:L-above-1}
With $\aZero=\frac{753}{1000}$ and $\theta_0=\frac34$, one has
\[
L_{\aZero}(\theta_0)>1.
\]
\end{lemma}

\begin{proof}
by Lemma \ref{lem:series}, we have
\begin{equation*}
L_{\alpha_0}(\theta_0)=\theta_0 e^{-2\theta_0}\sum_{n=0}^\infty \frac{(2\theta_0)^n}{n!\,(n+1-\alpha_0)}=\theta_0 e^{-\frac32}S(x_0).
\end{equation*}
where $x_0=2\theta_0=\frac32$ and $S(x_0)=\sum_{n=0}^\infty \frac{x_0^n}{n!\,(n+1-\alpha_0)}=\sum_{n=0}^\infty \frac{(2\theta_0)^n}{n!\,(n+1-\alpha_0)}$.

We first estimate $e^{-3/2}$ by Taylor expansion.
Apply Taylor's theorem with Lagrange remainder to $f(x)=e^{-x}$ at $x=\frac32$, expanded at $0$,
through degree $9$.  There exists $\xi\in(0,\frac32)$ such that
\begin{equation}\label{eq:taylor-9}
e^{-3/2}=\sum_{k=0}^{9}\frac{(-\frac32)^k}{k!}+\frac{e^{-\xi}}{10!}\Bigl(\frac32\Bigr)^{10}.
\end{equation}
The remainder is strictly positive, so
\begin{equation}\label{eq:exp-lower-poly}
e^{-3/2}>\sum_{k=0}^{9}\frac{(-\frac32)^k}{k!}.
\end{equation}
A direct rational computation gives
\[
\sum_{k=0}^{9}\frac{(-\frac32)^k}{k!}
=
\frac{511775}{2293760}
=
\frac{102355}{458752}.
\]

Therefore, we have
\[
e^{-3/2}>\frac{102355}{458752}.
\]

Next, we bound $S(x_0)$. All terms in the series are positive, so it is enough to keep the first six terms and conclude
\begin{equation}\label{eq:S-trunc}
S(x_0)>
\sum_{n=0}^{5}\frac{x_0^n}{n!\,(n+1-\aZero)}.
\end{equation}
Since $\aZero=\frac{753}{1000}$ and $x_0=\frac32$, the first six terms are
\begin{align*}
\frac{x_0^0}{0!(1-\aZero)}&=\frac{1000}{247},\\
\frac{x_0^1}{1!(2-\aZero)}&=\frac{1500}{1247},\\
\frac{x_0^2}{2!(3-\aZero)}&=\frac{375}{749},\\
\frac{x_0^3}{3!(4-\aZero)}&=\frac{1125}{6494},\\
\frac{x_0^4}{4!(5-\aZero)}&=\frac{3375}{67952},\\
\frac{x_0^5}{5!(6-\aZero)}&=\frac{225}{18656}.
\end{align*}
We now use the following six strict rational lower bounds:
\begin{align*}
\frac{1000}{247} &> \frac{2024}{500}, 
\\
\frac{1500}{1247} &> \frac{601}{500}, 
\\
\frac{375}{749} &> \frac{500}{999}, 
\\
\frac{1125}{6494} &> \frac{433}{2500}, 
\\
\frac{3375}{67952} &> \frac{149}{3000}, 
\\
\frac{225}{18656} &> \frac{301}{25000}. 
\end{align*}
Each is verified by cross-multiplication:
\begin{align*}
1000\cdot 500 &=500000>499928=2024\cdot 247,\\
1500\cdot 500 &=750000>749447=601\cdot 1247,\\
375\cdot 999 &=374625>374500=500\cdot 749,\\
1125\cdot 2500 &=2812500>2811902=433\cdot 6494,\\
3375\cdot 3000 &=10125000>10124848=149\cdot 67952,\\
225\cdot 25000 &=5625000>5615456=301\cdot 18656.
\end{align*}
Hence
\begin{equation}\label{eq:S-lb-sum}
S(x_0)
>
\frac{2024}{500}+\frac{601}{500}+\frac{500}{999}+\frac{433}{2500}+\frac{149}{3000}+\frac{301}{25000}.
\end{equation}
The right-hand side simplifies exactly to
\[
\frac{18685693}{3121875}.
\]
Finally,
\[
\frac{18685693}{3121875}>\frac{299}{50}
\iff
18685693\cdot 50>3121875\cdot 299,
\]
and indeed
\[
934284650>933440625.
\]
Therefore $S(x_0)>\frac{299}{50}$.

Since $\theta_0=\frac34>0$, combining all the estimates above, we obtain
\[
L_{\aZero}(\theta_0)
>
\frac34\cdot\frac{102355}{458752}\cdot\frac{299}{50}
=
\frac{91812435}{91750400}>1.
\]
This proves the lemma.
\end{proof}

\begin{proof}[Proof of Theorem~\ref{thm:alpha-upper}]

Let $\aZero=\frac{753}{1000}$ and $\theta_0=\frac34$. By Lemma~\ref{lem:L-above-1}, we have
\[
L_{\aZero}(\theta_0)>1.
\]
Now, consider arbitrary $K>0$ and $t_0>0$.
By Lemma~\ref{lem:Lminus-to-L}, the sequence $\Gamma_{\aZero,t}(\theta_0)$ is nondecreasing in $t$
and converges to $L_{\aZero}(\theta_0)$.  Since the limit is greater than 1, there exists an integer
$t\ge t_0$ such that
\begin{equation}\label{eq:choose-t-proof}
\Gamma_{\aZero,t}(\theta_0)>1.
\end{equation}
Now we set
\[
\rho=\rho_t(\theta_0)=\frac{1-e^{-2\theta_0/t}}{2}\in(0,\tfrac12).
\]
By Lemma~\ref{lem:Bt-lower}, we have
\[
t^{\aZero}B_t(\rho,\aZero)\ge \Gamma_{\aZero,t}(\theta_0)>1.
\]
Therefore
\begin{equation}\label{eq:Bt-base-bad}
B_t(\rho,\aZero)>\frac{1}{t^{\aZero}}.
\end{equation}
By monotonicity (Lemma~\ref{lem:B-monotone-alpha}), we also have
\[
t^\alpha B_t(\rho,\alpha)\ge t^{\aZero}B_t(\rho,\aZero)>1.
\]
Hence
\begin{equation}\label{eq:Bt-alpha-bad}
B_t(\rho,\alpha)>\frac{1}{t^\alpha}.
\end{equation}
Since $t+2>t$, we have $(t+2)^\alpha>t^\alpha$, and therefore
\[
\frac{1}{t^\alpha}>\frac{1}{(t+2)^\alpha}.
\]
Combining this with \eqref{eq:Bt-alpha-bad}, we get
\[
B_t(\rho,\alpha)>\frac{1}{(t+2)^\alpha}.
\]
Therefore, the necessary condition in Lemma~\ref{lem:reduce-to-B} cannot hold.

\end{proof}

To prove the lower bound for the matrix recursion $\Nb_t$, we first derive the following basic properties.
\begin{lemma}[Eigenvalue recursion]\label{prop:eig-rec}
Let
\[
\bar\M=\U\D\U^\top,
\qquad
\D=\diag(\rho_1,\dots,\rho_n),
\qquad
1\ge \rho_1\ge \cdots\ge \rho_n\ge 0.
\]
Define the following matrix recursion:
  \begin{align}
    \Nb_0 = \bar\M,\qquad \Nb_{t+1} = \Nb_t(\I-2\bar\M) + \|\Nb_t\|\cdot \bar\M.\label{eq:mat-rec2}
  \end{align}
Then the followings hold:

(i) Every $\Nb_t$ is of the form
\[
\Nb_t=\U\diag(\lambda_{1,t},\dots,\lambda_{n,t})\U^\top,
\]
with $\lambda_{k,0}=\rho_k$ and
\begin{equation}\label{eq:lambda-rec2}
\lambda_{k,t+1}=(1-2\rho_k)\lambda_{k,t}+\rho_k\mu_t,
\qquad \text{ where }\qquad 
\mu_t=\max_{1\le i\le n}\lambda_{i,t}=\norm{\Nb_t}.
\end{equation}

(ii) $\Nb_t\succeq 0$ for all $t\ge 0$.

(iii) For every $t\ge 0$,
\begin{equation}\label{eq:N-matrix-lb}
\Nb_t\succeq \bar\M(\I-\bar\M)^t.
\end{equation}
Consequently,
\begin{equation}\label{eq:N-norm-lb}
\norm{\Nb_t}\ge \max_{1\le k\le n}\rho_k(1-\rho_k)^t.
\end{equation}
\end{lemma}

\begin{proof}
We first prove (i). Because $\Nb_0=\bar\M$ is a polynomial in $\bar\M$, the claim is true at $t=0$.
Assume $\Nb_t=p_t(\bar\M)$ for some real polynomial $p_t$.  Then
\[
\Nb_{t+1}=p_t(\bar\M)(\I-2\bar\M)+\norm{\Nb_t}\,\bar\M,
\]
which is again a polynomial in $\bar\M$.  Thus, by induction, every $\Nb_t$ is a polynomial in $\bar\M$
and therefore diagonal in the eigenbasis of $\bar\M$.
Writing
\[
\Nb_t=\U\diag(\lambda_{1,t},\dots,\lambda_{n,t})\U^\top,
\]
and substituting this into \eqref{eq:mat-rec2}, we obtain
\[
\lambda_{k,t+1}=(1-2\rho_k)\lambda_{k,t}+\rho_k\norm{\Nb_t},
\qquad k=1,\dots,n.
\]
Since $\Nb_t$ is symmetric and diagonal in the basis $\U$, its operator norm is the maximum absolute
value of its eigenvalues.  We will prove below that $\Nb_t\succeq 0$ for all $t$, so the operator norm is
simply the largest eigenvalue:
\[
\norm{\Nb_t}=\max_k \lambda_{k,t}=\mu_t.
\]

To prove (ii), we observe that this is true for $t=0$.  Assuming $\Nb_t\succeq 0$, observe that
$\Nb_t$ commutes with $\bar\M$ and
\begin{align*}
\Nb_{t+1}
&=\Nb_t(\I-2\bar\M)+\norm{\Nb_t}\,\bar\M\\
&=\Nb_t(\I-\bar\M)+(\norm{\Nb_t}\,\I-\Nb_t)\bar\M.
\end{align*}
The commuting matrices $\Nb_t$, $\I-\bar\M$, $\norm{\Nb_t}\,\I-\Nb_t$, and $\bar\M$ are all positive
semidefinite, hence so are the products $\Nb_t(\I-\bar\M)$ and $(\norm{\Nb_t}\,\I-\Nb_t)\bar\M$.
Their sum is therefore positive semidefinite, so $\Nb_{t+1}\succeq 0$.

To prove (iii), we recall the recursion
\begin{align}
\Nb_{t+1}
&=\Nb_t(\I-\bar\M)+(\norm{\Nb_t}\,\I-\Nb_t)\bar\M. \label{eq:N-splitting}
\end{align}
The second term on the right-hand side of \eqref{eq:N-splitting} is positive semidefinite,
so
\begin{equation}\label{eq:N-one-step-lb}
\Nb_{t+1}\succeq \Nb_t(\I-\bar\M).
\end{equation}
Since $\Nb_0=\bar\M$, iterating \eqref{eq:N-one-step-lb} gives
\[
\Nb_t\succeq \bar\M(\I-\bar\M)^t,
\]
which is \eqref{eq:N-matrix-lb}.
Now diagonalize in the basis of $\bar\M$.
The eigenvalues of $\bar\M(\I-\bar\M)^t$ are exactly
$\rho_k(1-\rho_k)^t$, so \eqref{eq:N-matrix-lb} implies
\[
\lambda_{k,t}\ge \rho_k(1-\rho_k)^t,
\qquad k=1,\dots,n.
\]
Taking the maximum over $k$ yields \eqref{eq:N-norm-lb}.
\end{proof}

With all the previous preparations, we can now prove the lower bound for the matrix recursion $\Nb_t$.
\begin{theorem}[Lower bound for the matrix recursion]\label{prop:N-polynomial-family}
Let
\[
\aZero=\frac{753}{1000},
\qquad
\theta_0=\frac34,
\qquad
\rho_m=\frac{1-e^{-2\theta_0/m}}{2}
\quad (m\in\N).
\]
Then there exist integers $t_{\ref{prop:N-polynomial-family}}\ge 1$ and a constant $c_{\ref{prop:N-polynomial-family}}>0$ with the following property:
for every horizon $T\ge t_{\ref{prop:N-polynomial-family}}$, if one defines the diagonal matrix
\[
\bar\M^{(T)}:=\diag(\rho_1,\rho_2,\dots,\rho_T)\in\R^{T\times T}
\]
and lets $(\Nb_t^{(T)})_{t\ge 0}$ be the recursion \eqref{eq:mat-rec2} associated with $\bar\M^{(T)}$,
then
\begin{equation}\label{eq:N-polynomial-family}
\norm{\Nb_t^{(T)}}\ge \frac{c_{\ref{prop:N-polynomial-family}}}{(t+1)^{\aZero}}
\qquad\text{for every }0\le t\le T.
\end{equation}
\end{theorem}

\begin{proof}
By Lemmas~\ref{lem:Lminus-to-L} and \ref{lem:L-above-1}, there exists an integer \(t_1\ge 1\) such that
\begin{equation}\label{eq:Lminus-above-one-all-large}
\Gamma_{\aZero,t}(\theta_0)\ge 1
\qquad\text{for every } t\ge t_1.
\end{equation}

Fix \(T\ge t_1\), and write
\[
\mu_t=\norm{\Nb_t^{(T)}},
\qquad
\lambda_{k,t}=\lambda_{k,t}^{(T)}.
\]
We will prove that
\begin{equation}\label{eq:target-induction}
\mu_t\ge c_1 (t+1)^{-\aZero}
\qquad (0\le t\le T)
\end{equation}
for some constant \(c_1>0\) independent of \(T\).

The shift by \(+1\) is convenient for two reasons. First, the theorem is stated starting at \(t=0\). Second, in the induction step the estimate at time \(i-1\) becomes
\[
\mu_{i-1}\ge c_1 i^{-\aZero},
\]
which matches exactly the power appearing in \(B_t(\rho_t,\aZero)\).

By Lemma~\ref{prop:eig-rec}, the eigenvalues satisfy
\[
\lambda_{k,t+1}=(1-2\rho_k)\lambda_{k,t}+\rho_k\mu_t,
\]
and therefore, after iteration,
\begin{equation}\label{eq:lambda-unrolled}
\lambda_{k,t}
=
\rho_k(1-2\rho_k)^t
+
\rho_k\sum_{i=1}^t (1-2\rho_k)^{t-i}\mu_{i-1}.
\end{equation}

We first handle the finitely many times \(0\le t<t_1\) by choosing a sufficiently small constant $c_1>0$. More precisely, by Lemma~\ref{prop:eig-rec} (iii), we have
\[
\mu_t\ge \rho_1(1-\rho_1)^t
\qquad (t\ge 0).
\]
Hence
\[
c_1=\min_{0\le s<t_1}(s+1)^{\aZero}\rho_1(1-\rho_1)^s>0
\]
is well defined and independent of \(T\). Consequently,
\begin{equation}\label{eq:base-induction}
\mu_t\ge c_1 (t+1)^{-\aZero}
\qquad (0\le t<t_1).
\end{equation}

We now prove \eqref{eq:target-induction} for \(t\ge t_1\) by induction. Fix
\(t\in\{t_1,\dots,T\}\), and assume that
\[
\mu_s\ge c_1 (s+1)^{-\aZero}
\qquad (0\le s<t)
\]
or equivalently,
\[
\mu_{i-1}\ge c_1 i^{-\aZero}
\qquad (1\le i\le t).
\]

Using \eqref{eq:lambda-unrolled} with \(k=t\), dropping the nonnegative first term, and then using the induction hypothesis, we obtain
\[
\mu_t\ge \lambda_{t,t}
\ge
\rho_t\sum_{i=1}^t (1-2\rho_t)^{t-i}\mu_{i-1}
\ge
c_1\,\rho_t\sum_{i=1}^t (1-2\rho_t)^{t-i}i^{-\aZero}
=
c_1 B_t(\rho_t,\aZero).
\]

Now
\[
1-2\rho_t=e^{-2\theta_0/t},
\qquad\text{so}\qquad
\frac{t}{2}\bigl(-\log(1-2\rho_t)\bigr)=\theta_0.
\]
Therefore Lemma~\ref{lem:Bt-lower}, together with \eqref{eq:Lminus-above-one-all-large}, yields
\[
t^{\aZero}B_t(\rho_t,\aZero)\ge \Gamma_{\aZero,t}(\theta_0)\ge 1.
\]
Hence
\[
B_t(\rho_t,\aZero)\ge t^{-\aZero}\ge (t+1)^{-\aZero},
\]
and so
\[
\mu_t\ge c_1 (t+1)^{-\aZero}.
\]
which completes the induction.

We have thus shown that \eqref{eq:N-polynomial-family} holds for every \(0\le t\le T\). Therefore the theorem follows with
\[
t_{\ref{prop:N-polynomial-family}}=t_1,
\qquad
c_{\ref{prop:N-polynomial-family}}=c_1.
\]
\end{proof}

\end{document}